%% file: main.tex
\documentclass{article}

\usepackage{microtype}
\usepackage{graphicx}
\usepackage{subfigure}
\usepackage{booktabs} 

\usepackage{hyperref}

\usepackage[table,xcdraw]{xcolor}

\usepackage[accepted]{icml2020}

\usepackage{amsfonts}       
\usepackage{nicefrac}       
\usepackage{microtype}      
\usepackage{url}            

\usepackage{amsmath}
\DeclareMathOperator*{\argmax}{arg\,max}

\usepackage{mathtools}
\usepackage{mathrsfs}
\mathtoolsset{showonlyrefs}

\usepackage{amsthm}
\usepackage{thmtools,thm-restate}
\newtheorem{defn}{Definition}[]

\newtheorem{lemma}{Lemma}[]
\newtheorem{prop}{Property}[]

\icmltitlerunning{Evaluating the Performance of RL Algorithms}

\begin{document}

\twocolumn[
\icmltitle{Evaluating the Performance of Reinforcement Learning Algorithms}



\icmlsetsymbol{equal}{*}

\begin{icmlauthorlist}
\icmlauthor{Scott M.~Jordan}{umass}
\icmlauthor{Yash Chandak}{umass}
\icmlauthor{Daniel Cohen}{umass}
\icmlauthor{Mengxue Zhang}{umass}
\icmlauthor{Philip S.~Thomas}{umass}
\end{icmlauthorlist}

\icmlaffiliation{umass}{College of Information and Computer Sciences, University of Massachusetts, MA, USA}
\icmlcorrespondingauthor{Scott Jordan}{sjordan@cs.umass.edu}

\icmlkeywords{Reinforcement Learning, Performance Evaluation}

\vskip 0.3in
]


\printAffiliationsAndNotice{}  

\begin{abstract}
Performance evaluations are critical for quantifying algorithmic advances in reinforcement learning. Recent reproducibility analyses have shown that reported performance results are often inconsistent and difficult to replicate. In this work, we argue that the inconsistency of performance stems from the use of flawed evaluation metrics. Taking a step towards ensuring that reported results are consistent, we propose a new comprehensive evaluation methodology for reinforcement learning algorithms that produces reliable measurements of performance both on a single environment and when aggregated across environments. We demonstrate this method by evaluating a broad class of reinforcement learning algorithms on standard benchmark tasks.
\end{abstract}

\section{Introduction}
\label{sec:intro}

When applying reinforcement learning (RL), particularly to real-world applications, it is desirable to have algorithms that reliably achieve high levels of performance without requiring expert knowledge or significant human intervention. 
For researchers, having algorithms of this type would mean spending less time tuning algorithms to solve benchmark tasks and more time developing solutions to harder problems. 
Current evaluation practices do not properly account for the uncertainty in the results \citep{henderson2018deep} and neglect the difficulty of applying RL algorithms to a given problem. 
Consequently, existing RL algorithms are difficult to apply to real-world applications \citep{arnold2019realworld}. 
To both make and track progress towards developing reliable and easy-to-use algorithms, we propose a principled evaluation procedure that quantifies the difficulty of using an algorithm.

For an evaluation procedure to be useful for measuring the usability of RL algorithms, we suggest that it should have four properties. 
First, to ensure accuracy and reliability, an evaluation procedure should be \textit{scientific}, such that it provides information to answer a research question, tests a specific hypothesis, and quantifies any uncertainty in the results. 
Second, the performance metric captures the \textit{usability} of the algorithm over a wide variety of environments. For a performance metric to capture the usability of an algorithm, it should include the time and effort spent tuning the algorithm's hyperparameters (e.g., step-size and policy structure). 
Third, the evaluation procedure should be \textit{nonexploitative} \citep{balduzzi2018eval}, 
meaning no algorithm should be favored by performing well on an over-represented subset of environments or by abusing a particular score normalization method. 
Fourth, an evaluation procedure should be \textit{computationally tractable}, meaning that a typical researcher should be able to run the procedure and repeat experiments found in the literature.

As an evaluation procedure requires a question to answer, we pose the following to use throughout the paper: 
\textit{which algorithm(s) perform well across a wide variety of environments with little or no environment-specific tuning}? Throughout this work, we refer to this question as the \textit{general evaluation question}. 
This question is different from the one commonly asked in articles proposing a new algorithm, e.g., the \textit{common question} is, can algorithm X outperform other algorithms on tasks A, B, and C? 
In contrast to the common question, the expected outcome for the general evaluation question is not to find methods that maximize performance with optimal hyperparameters but to identify algorithms that do not require extensive hyperparameter tuning and thus are easy to apply to new problems.

In this paper, we contend that the standard evaluation approaches do not satisfy the above properties, and are not able to answer the general evaluation question. 
Thus, we develop a new procedure for evaluating RL algorithms that overcomes these limitations and can accurately quantify the uncertainty of performance. 
The main ideas in our approach are as follows. 
We present an alternative view of an algorithm such that sampling its performance can be used to answer the general evaluation question. 
We define a new normalized performance measure, \textit{performance percentiles}, which uses a relative measure of performance to compare algorithms across environments. 
We show how to use a game-theoretic approach to construct an aggregate measure of performance that permits quantifying uncertainty. 
Lastly, we develop a technique, \textit{performance bound propagation} (PBP), to quantify and account for uncertainty throughout the entire evaluation procedure. 
We provide source code so others may easily apply the methods we develop here.\footnote{Source code for this paper can be found at \url{https://github.com/ScottJordan/EvaluationOfRLAlgs}.} 

\section{Notation and Preliminaries}

In this section, we give notation used in this paper along with an overview of an evaluation procedure. 
In addition to this section, a list of symbols used in this paper is presented in Appendix \ref{app:confagg}.
We represent a performance metric of an algorithm, $i \in \mathcal{A}$, on an environment, $j \in \mathcal{M}$, as a random variable $X_{i,j}$. This representation captures the variability of results due to the choice of the random seed controlling the stochastic processes in the algorithm and the environment. 
The choice of the metric depends on the property being studied and is up to the experiment designer. The performance metric used in this paper is the average of the observed returns from one execution of the entire training procedure, which we refer to as the average return. 
The cumulative distribution function (CDF), $F_{X_{i,j}} \colon \mathbb{R} \to [0,1]$, describes the performance distribution of algorithm $i$ on environment $j$ such that $F_{X_{i,j}}(x) \coloneqq \Pr(X_{i,j} \le x)$. 
The quantile function, $Q_{X_{i,j}}(\alpha) \coloneqq \inf_{x} \{ x \in \mathbb{R} | F_{X_{i,j}}(x) \ge \alpha \}$, maps a cumulative probability, $\alpha \in (0,1)$, to a score such that $\alpha$ proportion of samples of $X_{i,j}$ are less than or equal to $Q_{X_{i,j}}(\alpha)$.
A normalization function, $g \colon \mathbb{R} \times \mathcal{M} \to \mathbb{R}$, maps a score, $x$, an algorithm receives on an environment, $j$, to a normalized score, $g(x,j)$, which has a common scale for all environments. 
In this work, we seek an aggregate performance measure, $y_i \in \mathbb{R}$, for an algorithm, $i$, such that $y_i \coloneqq \sum_{j=1}^{|\mathcal{M}|} q_j \mathbf{E}[g(X_{i,j},j)]$, where $q_j \ge 0$ for all $j \in \{1,\dotsc,|\mathcal{M}|\}$ and $\sum_{j=1}^{|\mathcal{M}|}q_j = 1$. In Section \ref{sec:aggregation}, we discuss choices for the normalizing function $g$ and weightings $q$ that satisfy the properties specified in the introduction.

The primary quantities of interest in this paper are the aggregate performance measures for each algorithm and confidence intervals on that measure. Let $y \in \mathbb{R}^{|\mathcal{A}|}$ be a vector representing the aggregate performance for each algorithm. We desire confidence intervals, $Y^-,Y^+ \in \mathbb{R}^{|\mathcal{A}|} \times \mathbb{R}^{|\mathcal{A}|}$, such that, for a confidence level $\delta \in (0,0.5]$, 
\begin{equation}
    \Pr \left (\forall i \in \{1,2,\dotsc,|\mathcal A|\}, y_i \in [Y^-_i,Y^+_i] \right) \ge 1 - \delta.
\end{equation}

To compute an aggregate performance measure and its confidence intervals that meet the criteria laid out in the introduction, one must consider the entire evaluation procedure. We view an evaluation procedure to have three main components: data collection, data aggregation, and reporting of the results. During the data collection phase, samples are collected of the performance metric $X_{i,j}$ for each combination $(i,j)$ of an algorithm $i \in \mathcal{A}$ and environment $j \in \mathcal{M}$. 
In the data aggregation phase, all samples of performance are normalized so the metric on each environment is on a similar scale, then they are aggregated to provide a summary of each algorithm's performance across all environments. Lastly, the uncertainty of the results is quantified and reported.

\section{Data Collection}
\label{sec:datacollection}

In this section, we discuss how common data collection methods are unable to answer the general evaluation question and then present a new method that can. 
We first highlight the core difference in our approach to previous methods.

The main difference between data collection methods is in how the samples of performance are collected for each algorithm on each environment.
Standard approaches rely on first tuning an algorithm's hyperparameters, i.e., any input to an algorithm that is not the environment, and then generating samples of performance. 
Our method instead relies on having a definition of an algorithm that can automatically select, sample, or adapt hyperparameters. This method can be used to answer the general evaluation question because its performance measure represents the knowledge required to use the algorithm. We discuss these approaches below.

\subsection{Current Approaches}
\label{sec:datacollectioncurrent}

\begin{figure}
    \centering
    \includegraphics[width=0.4\textwidth]{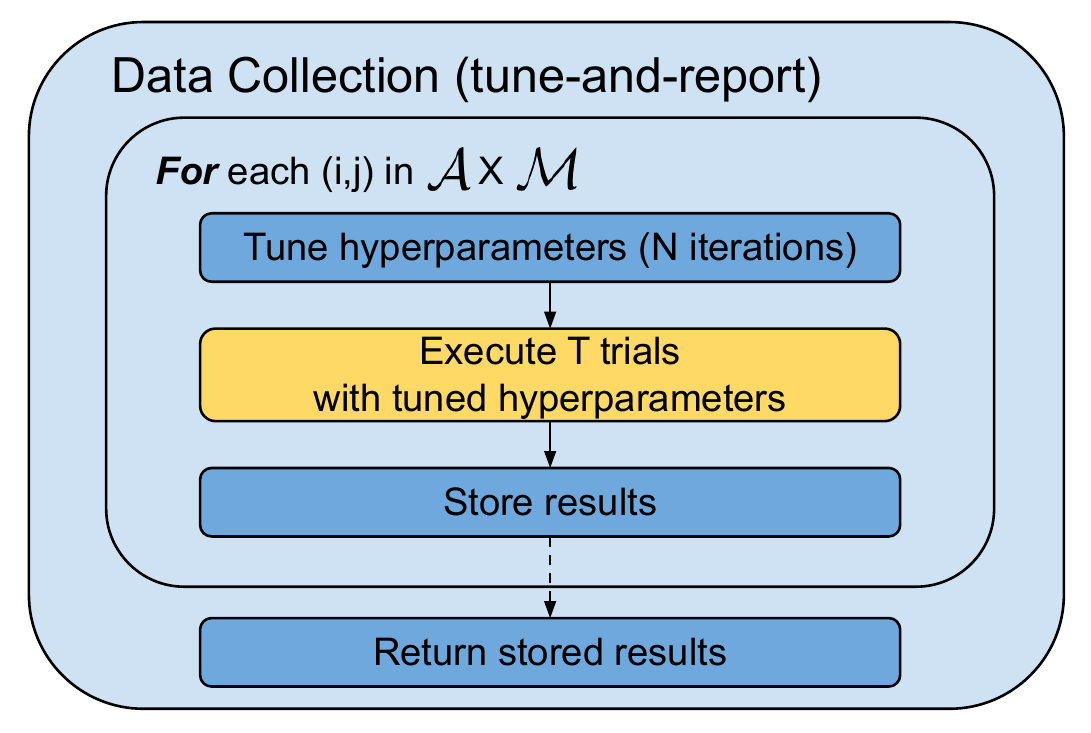}
    \caption{Data collection process of the tune-and-report method. The yellow box indicates trials using different random seeds.}
    \label{fig:persampling}
\end{figure}

A typical evaluation procedure used in RL research is the \textit{tune-and-report} method. As depicted in Figure \ref{fig:persampling}, the tune-and-report method has two phases: a tuning phase and a testing phase. In the tuning phase, hyperparameters are optimized either manually or via a hyperparameter optimization algorithm. Then after tuning, only the best hyperparameters are selected and executed for $T$ trials using different random seeds to provide an estimate of performance.

The tune-and-report data collection method does not satisfy the usability requirement or the scientific requirement. 
Recall that our objective is to capture the difficulty of using a particular algorithm. 
Because the tune-and-report method ignores the amount of data used to tune the hyperparameter, an algorithm that only works well after significant tuning could be favored over one that works well without environment-specific tuning, thus, violating the requirements.

Consider an extreme example of an RL algorithm that includes all policy parameters as hyperparameters. This algorithm would then likely be optimal after any iteration of hyperparameter tuning that finds the optimal policy. 
This effect is more subtle in standard algorithms, where hyperparameter tuning infers problem-specific information about how to search for the optimal policy, (e.g., how much exploration is needed, or how aggressive policy updates can be).  
Furthermore, this demotivates the creation of algorithms that are easier to use but do not improve performance after finding optimal hyperparameters.

The tune-and-report method violates the scientific property by not accurately capturing the uncertainty of performance. Multiple i.i.d.~samples of performance are taken after hyperparameter tuning and used to compute a bound on the mean performance. However, these samples of performance do not account for the randomness due to hyperparameter tuning. As a result, any statistical claim would be inconsistent with repeated evaluations of this method. 
This has been observed in several studies where further hyperparameter tuning has shown no difference in performance relative to baseline methods \citep{lucic2018gans,melis2018eval}.

The evaluation procedure proposed by \citet{dabney2014adaptive} addresses issues with uncertainty due to hyperparameter tuning and performance not capturing the usability of algorithms. Dabney's evaluation procedure computes performance as a weighted average over all $N$ iterations of hyperparameter tuning, and the entire tuning process repeats for $T$ trials. Even though this evaluation procedure fixes the problems with the tune-and-report approach, it violates our  computationally tractable property by requiring $TN$ executions of the algorithm to produce just $T$ samples of performance. 
In the case where $N=1$ it is not clear how hyperparameters should be set.
Furthermore, this style of evaluation does not cover the case where it is prohibitive to perform hyperparameter tuning, e.g., slow simulations, long agent lifetimes, lack of a simulator, and situations where it is dangerous or costly to deploy a bad policy. 
In these situations, it is desirable for algorithms to be insensitive to the choice of hyperparameters or able to adapt them during a single execution. It is in this setting that the general evaluation question can be answered.

\begin{figure}
    \centering
    \includegraphics[width=0.4\textwidth]{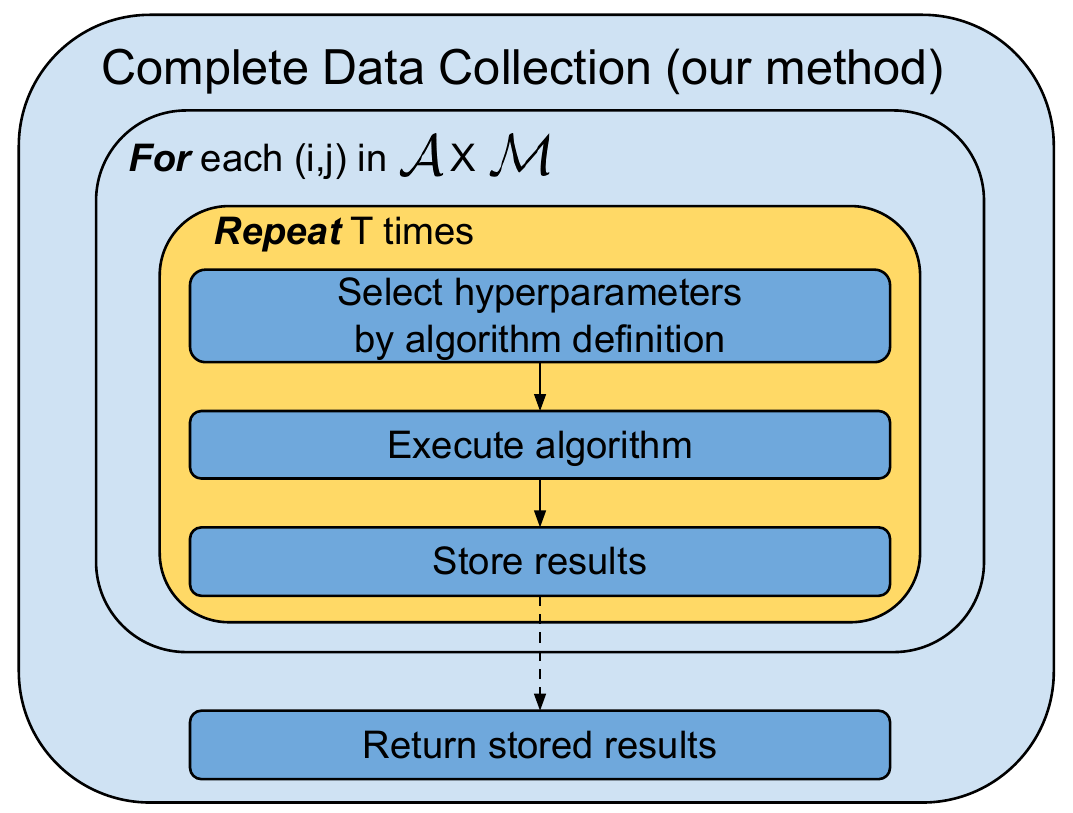}
    \caption{Data collection process using complete algorithm definitions. The yellow box indicates using different random seeds.}
    \label{fig:completedata}
\end{figure}

\subsection{Our Approach}
\label{sec:datacollectionnew}

In this section, we outline our method, \textit{complete data collection}, that does not rely on hyperparameter tuning.
If there were no hyperparameters to tune, evaluating algorithms would be simpler. Unfortunately, how to automatically set hyperparameters has been an understudied area. Thus, we introduce the notion of a complete algorithm definition. 

\begin{defn}[Algorithm Completeness]
An algorithm is complete on an environment $j$, when defined such that the only required input to the algorithm is meta-information about environment $j$, e.g., the number of state features and actions.
\end{defn}

Algorithms with a complete definition can be used on an environment and without specifying any hyperparameters. 
Note that this does not say that an algorithm cannot receive forms of problem specific knowledge, only that it is not required. 
A well-defined algorithm will be able to infer effective combinations of hyperparameters or adapt them during learning. 
There are many ways to make an existing algorithm complete. In this work, algorithms are made complete by defining a distribution from which to randomly sample hyperparameters.
Random sampling may produce poor or divergent behavior in the algorithm, but this only indicates that it is not yet known how to set the hyperparameters of the algorithm automatically. Thus, when faced with a new problem, finding decent hyperparameters will be challenging.
One way to make an adaptive complete algorithm is to include a hyperparameter optimization method in the algorithm. However, all tuning must be done within the same fixed amount of time and cannot propagate information over trials used to obtain statistical significance.

Figure \ref{fig:completedata} shows the complete data collection method. For this method we limit the scope of algorithms to only include ones with complete definitions; thus, it does not violate any of the properties specified. This method satisfies the scientific requirement since it is designed to answer the general evaluation question, and the uncertainty of performance can be estimated using all of the trials. Again, this data collection method captures the difficulty of using an algorithm since the complete definition encodes the knowledge necessary for the algorithm to work effectively. 
The compute time of this method is tractable, since $T$ executions of the algorithm produces $T$ independent samples of performance.

The practical effects of using the complete data collection method are as follows. 
Researchers do not have to spend time tuning each algorithm to try and maximize performance. 
Fewer algorithm executions are required to obtain a statistically meaningful result. 
With this data collection method, improving upon algorithm definitions will become significant research contributions and lead to algorithms that are easy to apply to many problems.

\section{Data Aggregation}
\label{sec:aggregation}

Answering the general evaluation question requires a ranking of algorithms according to their performance on all environments $\mathcal{M}$. The aggregation step accomplishes this task by combining the performance data generated in the collection phase and summarizing it across all environments. 
However, data aggregation introduces several challenges.
First, each environment has a different range of scores that need to be normalized to a common scale. 
Second, a uniform weighting of environments can introduce bias.
For example, the set of environments might include many slight variants of one domain, giving that domain a larger weight than a single environment coming from a different domain.

\subsection{Normalization}
\label{sec:normalization}
The goal in score normalization is to project scores from each environment onto the same scale while not being exploitable by the environment weighting.
In this section, we first show how existing normalization techniques are exploitable or do not capture the properties of interest. Then we present our normalization technique: performance percentiles.

\subsubsection{Current Approaches}

We examine two normalization techniques: performance ratios and policy percentiles. We discuss other normalization methods in Appendix \ref{app:normalization}.
The performance ratio is commonly used with the Arcade Learning Environment to compare the performance of algorithms relative to human performance \cite{mnih2015dqn,machado2018ale}.
The performance ratio of two algorithms $i$ and $k$ on an environment $j$ is $\mathbf{E}[X_{i,j}]/\mathbf{E}[X_{k,j}]$. 
This ratio is sensitive to the location and scale of the performance metric on each environment, such that an environment with scores in the range $[0,1]$ will produce larger differences than those on the range $[1000, 1001]$. Furthermore, all changes in performance are assumed to be equally challenging, i.e., going from a score of $0.8$ to $0.89$ is the same difficulty as $0.9$ to $0.99$. This assumption of linearity of difficulty is not reflected on environments with nonlinear changes in the score as an agent improves, e.g., completing levels in Super Mario.

A critical flaw in the performance ratio is that it can produce an arbitrary ordering of algorithms when combined with the arithmetic mean, $\sum_j q_j \mathbf{E}[X_{i,j}] / \mathbf{E}[X_{k,j}]$ \citep{fleming1986geom},
meaning a different algorithm in the denominator could change the relative rankings. 
Using the geometric mean can address this weakness of performance ratios, but does not resolve the other issues.

Another normalization technique is \textit{policy percentiles}, a method that projects the score of an algorithm through the performance CDF of random policy search \citep{dabney2014adaptive}. 
The normalized score for an algorithm, $i$, is $F_{X_{\Pi,j}}(X_{i,j})$, where $F_{X_{\Pi,j}}$ is the performance CDF when a policy is sampled uniformly from a set of policies, $\Pi$, on an environment $j$, i.e, $\pi \sim U(\Pi)$. 
Policy percentiles have a unique advantage in that performance is scaled according to how difficult it is to achieve that level of performance relative to random policy search. 
Unfortunately, policy percentiles rely on specifying $\Pi$, which often has a large search space. As a result, most policies will perform poorly, making all scores approach $1.0$. It is also infeasible to use when random policy search is unlikely to achieve high levels of performance. 
Despite these drawbacks, the scaling of scores according to a notion of difficulty is desirable, so we adapt this idea to use any algorithm's performance as a reference distribution.

\subsubsection{Our Approach}
\label{sec:normalizationours}

An algorithm's performance distribution can have an interesting shape with large changes in performance that are due to divergence, lucky runs, or simply that small changes to a policy can result in large changes in performance \citep{jordan2018cdf}.
These effects can be seen in Figure \ref{fig:cdf_cartpole}, where there is a quick rise in cumulative probability for a small increase in performance. 
Inspired by \citet{dabney2014adaptive}'s policy percentiles, we propose \textit{performance percentiles}, a score normalization technique that can represent these intricacies.

The probability integral transform shows that projecting a random variable through its CDF transforms the variable to be uniform on $[0,1]$ \citep{dodge2006oxford}. Thus, normalizing an algorithm's performance by its CDF will equally distribute and represent a linear scaling of difficulty across $[0,1]$. When normalizing performance against another algorithm's performance distribution, the normalized score distribution will shift towards zero when the algorithm is worse than the normalizing distribution and shift towards one when it is superior. 
As seen in Figure \ref{fig:cdf_cartpole}, the CDF can be seen as encoding the relative difficulty of achieving a given level of performance, where large changes in an algorithm's CDF output indicate a high degree of difficulty for that algorithm to make an improvement and similarly small changes in output correspond to low change in difficulty.
In this context difficulty refers to the amount of random chance (luck) needed to achieve a given level of performance.

To leverage these properties of the CDF, we define performance percentiles, that use a weighted average of each algorithm's CDF to normalize scores for each environment. 
\begin{defn}[Performance Percentile]
\label{def:performancepercentile}
In an evaluation of algorithms, $\mathcal{A}$, the performance percentile for a score $x$ on an environment, $j$, is $F_{\bar X_j}(x, w_j)$, where $F_{\bar X_j}$ is the mixture of CDFs  $F_{\bar X_j}(x,w_j) \coloneqq \sum_{i=1}^{|\mathcal{A}|} w_{j,i} F_{X_{i,j}}(x)$, with weights $w_j \in \mathbb{R}^{|\mathcal{A}|}$, 
$\sum_{i=1}^{|\mathcal{A}|} w_{j,i} = 1$, and $\forall i \ w_{j,i} \ge 0$.
\end{defn}
So we can say that performance percentiles capture the performance characteristic of an environment relative to some averaged algorithm. We discuss how to set the weights $w_j$ in the next section.

\begin{figure}
    \centering
    \includegraphics[width=0.48\textwidth]{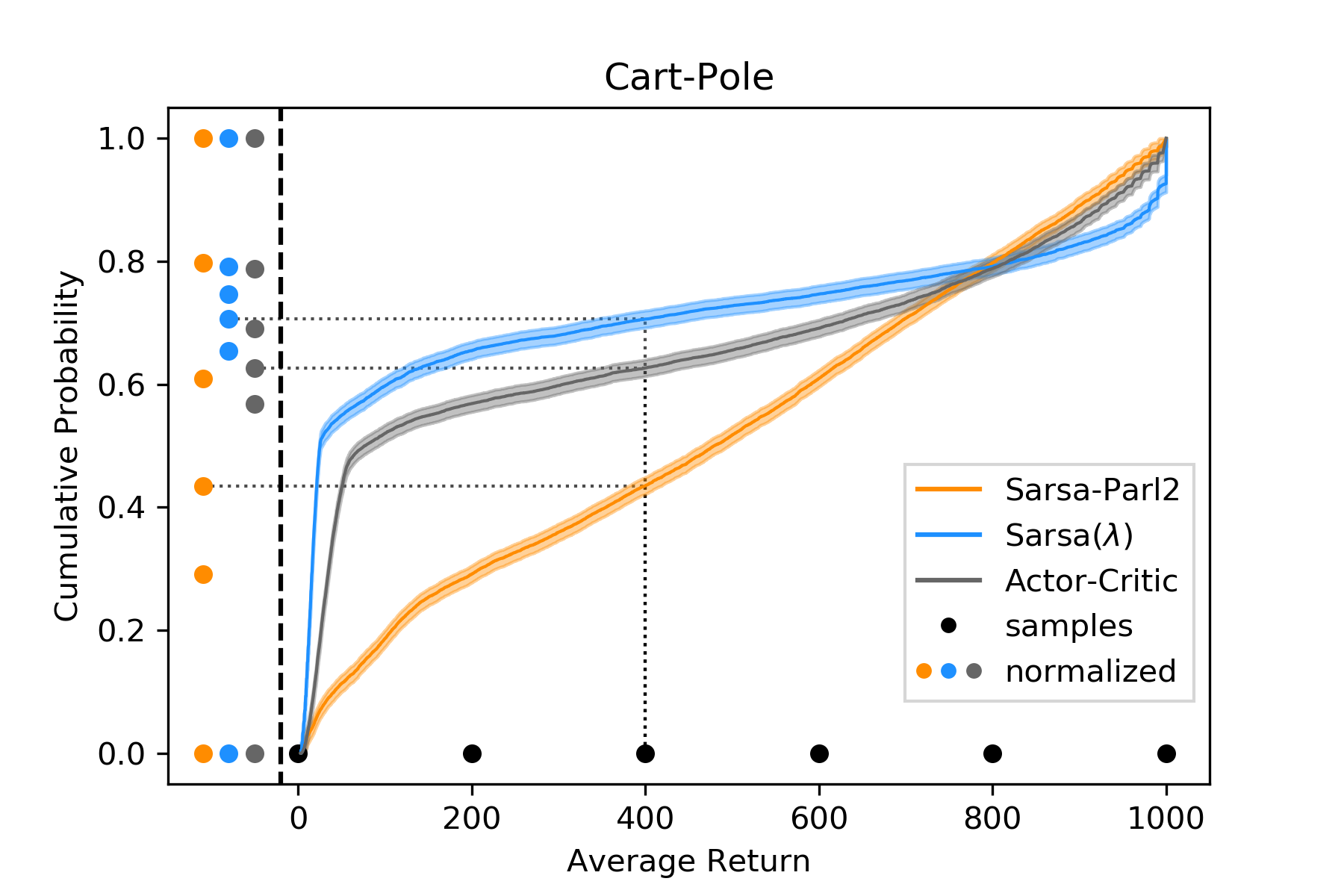}
    \caption{This plot shows the CDF of average returns for the Sarsa-Parl2, Sarsa($\lambda$), and Actor-Critic algorithms on the Cart-Pole environment. Each line represents the empirical CDF using $10,\!000$ trials and the shaded regions represent the $95\%$ confidence intervals. To illustrate how the performance percentiles work, this plot shows how samples of performance (black dots) are normalized by each CDF, producing the normalized scores (colored dots). The correspondence between a single sample and its normalized score is shown by the dotted line.}
    \label{fig:cdf_cartpole}
\end{figure}

Performance percentiles are closely related to the concept of (probabilistic) performance profiles \cite{Dolan2002profiles,barreto2010probeval}. The difference being that performance profiles report the cumulative distribution of normalized performance metrics over a set of tasks (environments), whereas performance percentiles are a technique for normalizing scores on each task (environment).

\subsection{Summarization}
\label{sec:summarization}

A weighting over environments is needed to form an aggregate measure. 
We desire a weighting over environments such that no algorithm can exploit the weightings to increase its ranking. Additionally, for the performance percentiles, we need to determine the weighting of algorithms to use as the reference distribution. Inspired by the work of \citet{balduzzi2018eval}, we propose a weighting of algorithms and environments, using the equilibrium of a two-player game.

In this game, one player, $p$, will try to select an algorithm to maximize the aggregate performance, while a second player, $q$, chooses the environment and reference algorithm to minimize $p$'s score. Player $p$'s pure strategy space, $\mathcal{S}_1$, is the set of algorithms $\mathcal{A}$, i.e., $p$ plays a strategy $s_1=i$ corresponding to an algorithm $i$. Player $q$'s pure strategy space, $\mathcal{S}_2$, is the cross product of a set of environments, $\mathcal{M}$, and algorithms, $\mathcal{A}$, i.e., player $q$ plays a strategy $s_2=(j,k)$ corresponding to a choice of environment $j$ and normalization algorithm $k$.
We denote the pure strategy space of the game by $\mathcal{S} \coloneqq \mathcal{S}_1 \times \mathcal{S}_2$. A strategy, $s \in \mathcal{S}$, can be represented by a tuple $s=(s_1,s_2)=(i,(j,k))$.

The utility of strategy $s$ is measured by a payoff function $u_p \colon \mathcal{S} \to \mathbb{R}$ and $u_q \colon \mathcal{S} \to \mathbb{R}$ for players $p$ and $q$ respectively. The game is defined to be zero sum, i.e., $u_q(s) = -u_p(s)$. We define the payoff function to be 
$u_p(s) \coloneqq \mathbf{E}[F_{X_{k,j}}(X_{i,j})]$.
Both players $p$ and $q$ sample strategies from probability distributions $p \in \Delta(\mathcal{S}_1)$ and $q \in \Delta(\mathcal{S}_2)$, where $\Delta(\mathcal{X})$ is the set of all probability distributions over $\mathcal{X}$.

The equilibrium solution of this game naturally balances the normalization and environment weightings to counter each algorithm's strengths without conferring an advantage to a particular algorithm. 
Thus, the aggregate measure will be useful in answering the general evaluation question.

After finding a solution $(p^*,q^*)$, the aggregate performance measure $y_i$ for an algorithm $i$ defined as
\begin{equation}
\label{eq:aggforalgi}
y_i \coloneqq \sum_{j=1}^{|\mathcal{M}|}\sum_{k=1}^{|\mathcal{A}|} q^*_{j,k} \mathbf{E}[F_{X_{k,j}}(X_{i,j})].
\end{equation}

To find a solution $(p^*,q^*)$, we employ the $\alpha$-Rank technique \citep{omidshafiei2019alpharank}, which returns a stationary distribution over the pure strategy space $\mathcal{S}$. $\alpha$-Rank allows for efficient computation of both the equilibrium and confidence intervals on the aggregate performance \citep{rowland2019alphabounds}. We detail this method and details of our implementation in Appendix \ref{app:alpharank}.

\section{Reporting Results}
\label{sec:reportingresults}

As it is crucial to quantify the uncertainty of all claimed performance measures, we first discuss how to compute confidence intervals for both single environment and aggregate measures, then give details on displaying the results.

\subsection{Quantifying Uncertainty}
\label{sec:aggbound}
In keeping with our objective to have a scientific evaluation, we require our evaluation procedure to quantify any uncertainty in the results. When concerned with only a single environment, standard concentration inequalities can compute confidence intervals on the mean performance. Similarly, when displaying the distribution of performance, one can apply standard techniques for bounding the empirical distribution of performance. 
However, computing confidence intervals on the aggregate has additional challenges.

Notice that in \eqref{eq:aggforalgi} computing the aggregate performance requires two unknown values: $q^*$ and the mean normalized performance, $\mathbf{E}[F_{X_{k,j}}(X_{i,j})]$. 
Since $q^*$ depends on mean normalized performance, any uncertainty in the mean normalized performance results in uncertainty in $q^*$. 
To compute valid confidence intervals on the aggregate performance, the uncertainty through the entire process must be considered.

We introduce a process to compute the confidence intervals, which we refer to as \textit{performance bound propagation} (PBP). 
We represent PBP as a function $\text{PBP} \colon \mathcal{D} \times \mathbb{R} \to \mathbb{R}^{|\mathcal{A}|} \times \mathbb{R}^{|\mathcal{A}|}$, which maps a dataset $D \in \mathcal{D}$ containing all samples of performance and a confidence level $\delta \in (0,0.5]$, to vectors $Y^-$ and $Y^+$ representing the lower and upper confidence intervals, i.e., $(Y^-, Y^+) = \text{PBP}(D,\delta)$.

The overall procedure for PBP is as follows, first compute confidence intervals for each $F_{X_{i,j}}$, then using these intervals compute confidence intervals on each mean normalized performance, next determine an uncertainty set $\mathcal{Q}$ for $q^*$ that results from uncertainty in the mean normalized performance, finally for each algorithm find the minimum and maximum aggregate performance over the uncertainty in the mean normalized performances and $\mathcal{Q}$. 
We provide pseudocode in Appendix \ref{app:confagg} and source code in the repository.

We prove that PBP produces valid confidence intervals for a confidence level $\delta \in (0, 0.5]$ and a dataset $D$ containing $T_{i,j} >1$ samples of performance for all algorithms $i \in \mathcal{A}$ and environments $j \in \mathcal{M}$.
\begin{restatable}{thm}{pbpthm}
If $(Y^-,Y^+) = \text{PBP}(D,\delta)$, then
\begin{equation}
  \Pr \left ( \forall i \in {1,2,\dotsc,|\mathcal{A}|},\  y_i \in [Y^-_i, Y^+_i] \right ) \ge 1-\delta.
\end{equation}
\end{restatable}
\begin{proof}
Although the creation of valid confidence intervals is critical to this contribution, due to space restrictions it is presented in Appendix \ref{app:confagg}.
\end{proof}

\subsection{Displaying Results}

\begin{figure}
    \centering
    \includegraphics[width=0.48\textwidth]{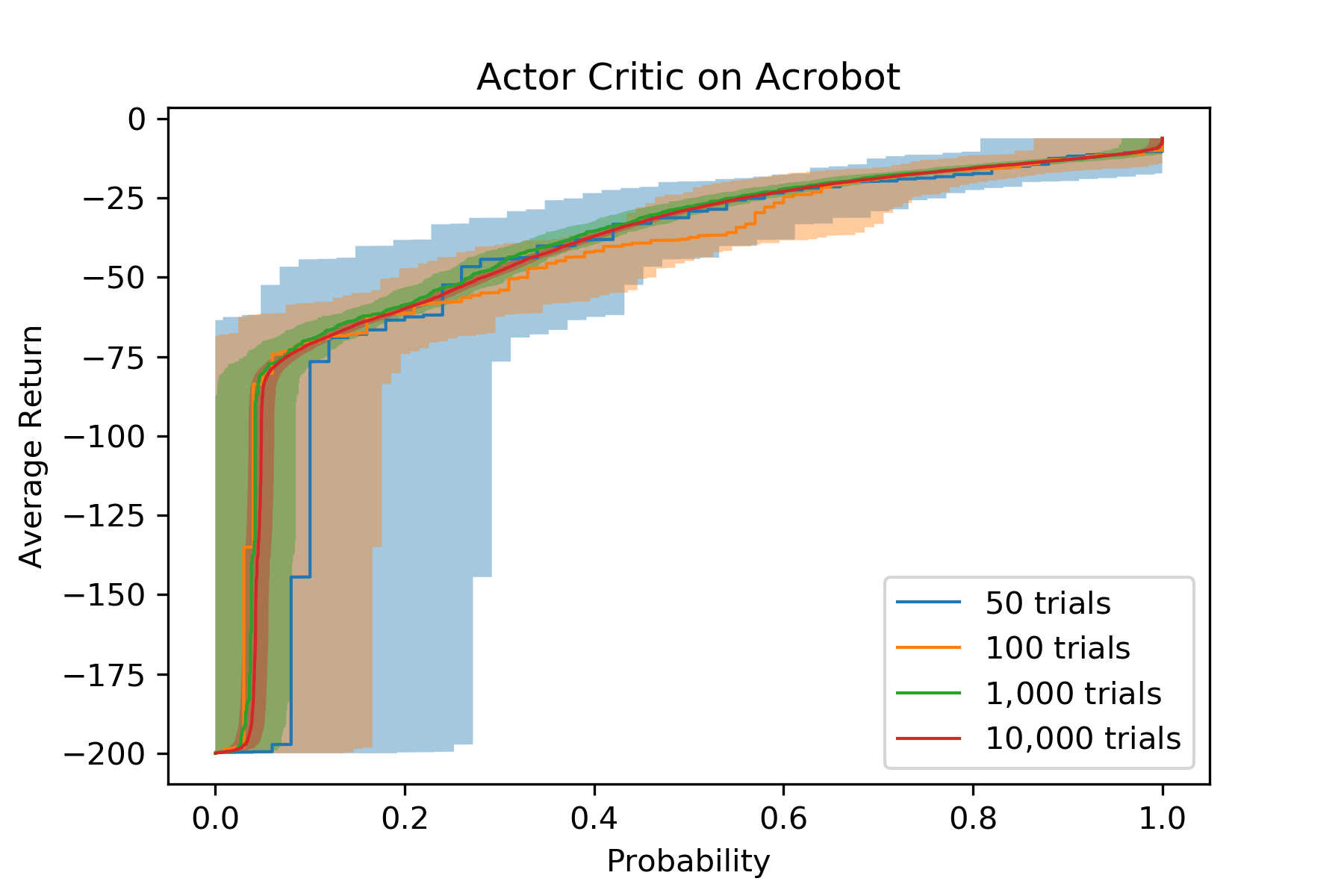}
    \caption{This plot shows the distribution of average returns for the Actor-Critic algorithm on the Acrobot environment. The $x$-axis represents a probability and the $y$-axis represents the average return such that the proportion of trials that have a value less than or equal to $y$ is $x$, e.g., at $x=0.5$, $y$ is the median. Each line represents the empirical quantile function using a different number of trials and the shaded regions represent the $95\%$ confidence intervals computed using concentration inequalities. 
    In this plot, the larger the area under the curve, the better the performance.
    This plot highlights the large amount of uncertainty when using small sample sizes and how much it decreases with more samples. }
    \label{fig:return_ab}
\end{figure}

In this section, we describe our method for reporting the results. There are three parts to our method: answering the stated hypothesis, providing tables and plots showing the performance and ranking of algorithms for all environments, and the aggregate score, then for each performance measure, provide confidence intervals to convey uncertainty.

The learning curve plot is a standard in RL and displays a performance metric (often the return) over regular intervals during learning. 
While this type of plot might be informative for describing some aspects of the algorithm's performance, it does not directly show the performance metric used to compare algorithms, making visual comparisons less obvious. 
Therefore, to provide the most information to the reader, we suggest plotting the distribution of performance for each algorithm on each environment. Plotting the distribution of performance has been suggested in many fields as a means to convey more information, \citep{Dolan2002profiles,farahmand2010cdfevo,reimers2017scoredist,cohen2018distributed}. 
Often in RL, the object is to maximize a metric, so we suggest showing the quantile function over the CDF as it allows for a more natural interpretation of the performance, i.e., the higher the curve, the better the performance \citep{bellemare13arcade}. Figure \ref{fig:return_ab} show the performance distribution with $95\%$ confidence intervals for different sample sizes. 
It is worth noting that when tuning hyperparameters the data needed to compute these distributions is already being collected, but only the results from the tuned runs are being reported. 
By only reporting only the tuned performance it shows what an algorithm can achieve not what it is likely to achieve.

\section{Experimental Results}
\label{sec:experiments}
In this section, we describe and report the results of experiments to illustrate how this evaluation procedure can answer the general evaluation question and identify when a modification to an algorithm or its definition improves performance. We also investigate the reliability of different bounding techniques on the aggregate performance measure.

\subsection{Experiment Description}
To demonstrate the evaluation procedure we compare the algorithms: Actor-Critic with eligibility traces (AC) \cite{sutton2018reinforcement}, Q($\lambda$), Sarsa($\lambda$), \cite{suttonbaro1998book}, NAC-TD \citep{morimura2005ntd,degris2012inac,thomas2014bias}, and proximal policy optimization (PPO) \citep{schulman2017ppo}. 
The learning rate is often the most sensitive hyperparameter in RL algorithms. 
So, we include three versions of Sarsa$(\lambda)$, Q$(\lambda)$, and AC: a base version, a version that scales the step-size with the number of parameters (e.g., Sarsa$(\lambda)$-s), and an adaptive step-size method, Parl2 \citep{dabney2014adaptive}, that does not require specifying the step size. 
Since none of these algorithms have an existing complete definition, we create one by randomly sampling hyperparameters from fixed ranges. We consider all parameters necessary to construct each algorithm, e.g., step-size, function approximator, discount factor, eligibility trace decay. For the continuous state environments, each algorithm employs linear function approximation using the Fourier basis \cite{konidaris2011fourier} with a randomly sampled order. See Appendix \ref{app:definitions} for full details of each algorithm.

These algorithms are evaluated on $15$ environments, eight discrete MDPs, half with stochastic transition dynamics, and seven continuous state environments: Cart-Pole \citep{florian2007cartpole}, Mountain Car \citep{suttonbaro1998book}, Acrobot \citep{sutton1995acrobot}, and four variations of the pinball environment \citep{konidaris2009pinball,geramifard2015rlpy}. For each independent trial, the environments have their dynamics randomly perturbed to help mitigate environment overfitting \citep{whiteson2011eval}; see code for details. For further details about the experiment see Appendix \ref{app:environments}.

While these environments have simple features compared to the Arcade Learning Environment \citep{bellemare13arcade}, they remain useful in evaluating RL algorithms for three reasons. First is that experiments finish quickly. Second, the environments provide interesting insights into an algorithm's behavior. Third, as our results will show, there is not yet a complete algorithm that can reliably solve each one.

We execute each algorithm on each environment for $10,\!000$ trials. While this number of trials may seem excessive, our goal is to detect a statistically meaningful result. Detecting such a result is challenging because the variance of RL algorithms performance is high; we are comparing $|\mathcal{A}| \times |\mathcal{M}| = 165$ random variables, and we do not assume the performances are normally distributed. Computationally, executing ten thousand trials is not burdensome if one uses an efficient programming language such as Julia \citep{bezanson2017julia} or C++, where we have noticed approximately two orders of magnitude faster execution than similar Python implementations. We investigate using smaller sample sizes at the end of this section.

\begin{table}
\begin{tabular}{lcr}
\toprule 
\multicolumn{3}{c}{\textbf{Aggregate Performance}} \\
\midrule
\textbf{Algorithms}         & \multicolumn{1}{c}{\textbf{Score}} & \textbf{Rank} \\ \midrule 
\rowcolor[HTML]{EFEFEF} 
Sarsa-Parl2        & $0.4623 (\textbf{0.3904}, 0.5537)$       & 1 (2,1)         \\
Q-Parl2            & $0.4366 (\textbf{0.3782}, 0.5632)$       & 2 (2,1)        \\
\rowcolor[HTML]{EFEFEF} 
AC-Parl2           & $0.1578 (0.0765, \textbf{0.3129})$       & 3 (11,3)        \\
Sarsa$(\lambda)$-s & $0.0930 (0.0337, 0.2276)$                & 4 (11,3)        \\
\rowcolor[HTML]{EFEFEF} 
AC-s               & $0.0851 (0.0305, 0.2146)$                & 5 (11,3)        \\
Sarsa$(\lambda)$   & $0.0831 (0.0290, 0.2019)$                & 6 (11,3)        \\
\rowcolor[HTML]{EFEFEF} 
AC                 & $0.0785 (0.0275, 0.2033)$                & 7 (11,3)        \\
Q$(\lambda)$-s     & $0.0689 (0.0237, 0.1973)$                & 8 (11,3)        \\
\rowcolor[HTML]{EFEFEF} 
Q$(\lambda)$       & $0.0640 (0.0214, 0.1780)$                & 9 (11,3)        \\
\rowcolor[HTML]{EFEFEF} 
NAC-TD             & $0.0516 (0.0180, 0.1636)$                & 10 (11,3)       \\
PPO                & $0.0508 (0.0169, 0.1749)$                & 11 (11,3)       \\
\bottomrule 
\end{tabular}
\caption{Aggregate performance measures for each algorithm and their rank. The parentheses contain the intervals computed using PBP and together all hold with $95\%$ confidence. The bolded numbers identify the best ranked statistically significant differences.}
\label{tab:agg_res}
\end{table}

\subsection{Algorithm Comparison}

\begin{figure}
    \centering
    \includegraphics[width=0.48\textwidth]{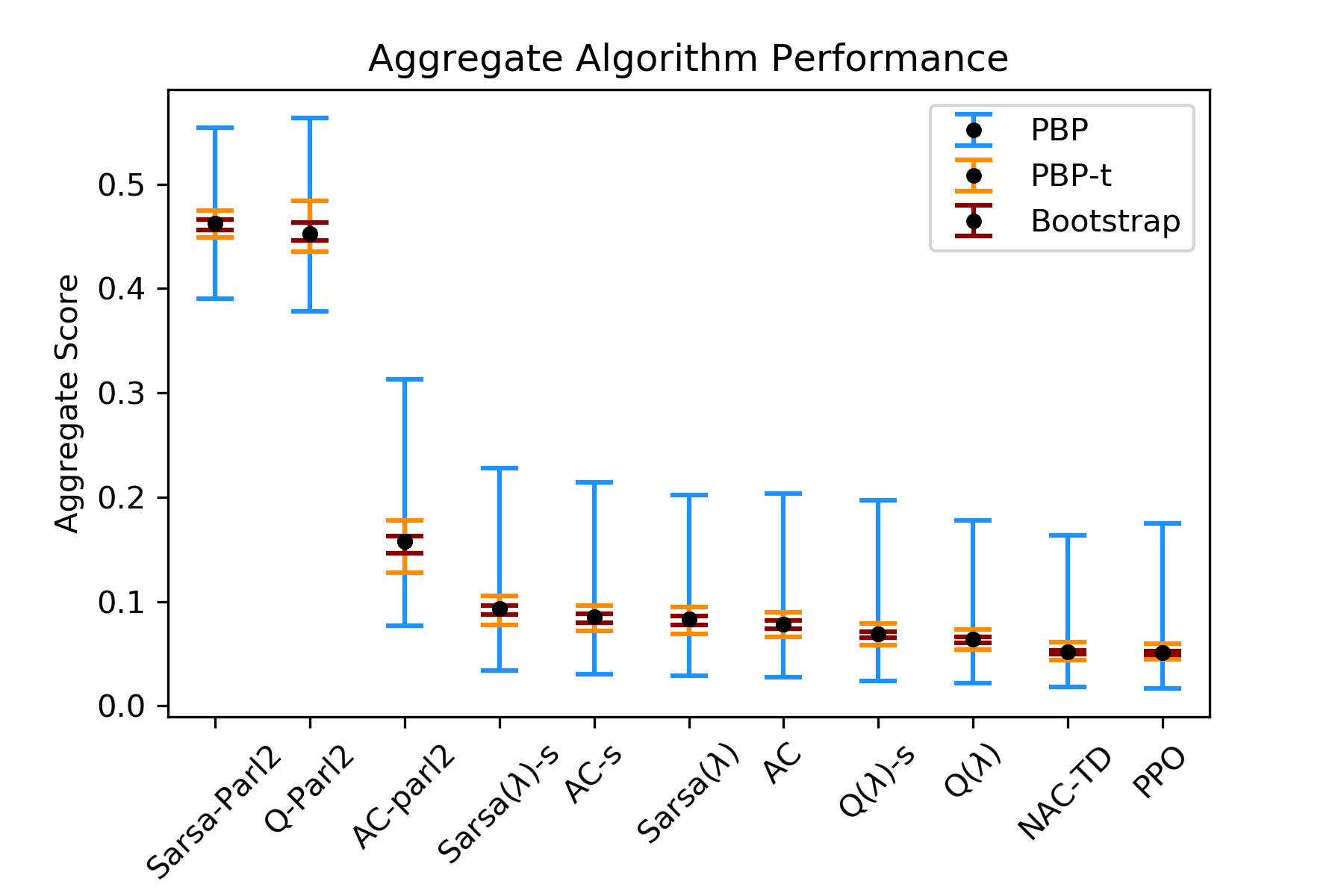}
    \caption{The aggregate performance for each algorithm with confidence intervals using PBP, PBP-t, and bootstrap. The width of each interval is scaled so all intervals hold with $95\%$ confidence.}
    \label{fig:aggperf}
\end{figure}

The aggregate performance measures and confidence intervals are illustrated in Figure \ref{fig:aggperf} and given in Table \ref{tab:agg_res}. Appendix \ref{app:indperf} lists the performance tables and distribution plots for each environment. Examining the empirical performances in these figures, we notice two trends. The first is that our evaluation procedure can identify differences that are not noticeable in standard evaluations. For example, all algorithms perform near optimally when tuned properly (indicated by the high end of the performance distribution). The primary differences between algorithms are in the frequency of high performance and divergence (indicated by low end of the performance distribution). Parl2 methods rarely diverge, giving a large boost in performance relative to the standard methods.

The second trend is that our evaluation procedure can identify when theoretical properties do or do not make an algorithm more usable. For example, Sarsa($\lambda$) algorithms outperform their Q($\lambda$) counterparts. This result \emph{might} stem from the fact that Sarsa($\lambda$) is known to converge with linear function approximation \citep{perkins2002linearsarsa} while Q($\lambda$) is known to diverge \citep{baird1995residual,wiering2004divql}. Additionally, NAC-TD performs worse than AC despite that natural gradients are a superior ascent direction. This result is due in part because it is unknown how to set the three step-sizes in NAC-TD, making it more difficult to use than AC. Together these observations point out the deficiency in the way new algorithms have been evaluated. That is, tuning hyperparameters hides the lack of knowledge required to use the algorithm, introducing bias that favors the new algorithm. In contrast, our method forces this knowledge to be encoded into the algorithm, leading to a more fair and reliable comparison.

\subsection{Experiment Uncertainty}
 While the trends discussed above might hold true in general, we must quantify our uncertainty. Based on the confidence intervals given using PBP, we claim with $95\%$ confidence that on these environments and according to our algorithm definitions, Sarsa-Parl2 and Q-Parl2 have a higher aggregate performance of average returns than all other algorithms in the experiment.
It is clear that $10,\!000$ trials per algorithm per environment is not enough to detect a unique ranking of algorithms using the nonparametric confidence intervals in PBP.
We now consider alternative methods, PBP-t, and the percentile bootstrap.
PBP-t replaces the nonparameteric intervals in PBP with ones based on the Student's t-distribution.
We detail these methods in Appendix \ref{app:altbounds}.
From Figure \ref{fig:aggperf}, it is clear that both alternative bounds are tighter and thus useful in detecting differences. 
Since assumptions of these bounds are different and not typically satisfied, it is unclear if they are valid.

\begin{table}
\begin{tabular}{|l|lr|lr|lr|}
\hline
          \multicolumn{7}{|c|}{\textbf{Confidence Interval Performance}} \\ \hline
           & \multicolumn{2}{c|}{PBP}  & \multicolumn{2}{c|}{PBP-t}  & \multicolumn{2}{c|}{Bootstrap}       \\ \hline
Samples    & FR & SIG & FR & SIG & FR & SIG \\ \hline
\rowcolor[HTML]{EFEFEF} 
$10$       & 0.0 & 0.0                 & 1.000 & 0.00               & 0.112 & 0.11 \\
$30$       & 0.0 & 0.0                 & 0.000 & 0.00               & 0.092 & 0.37  \\
\rowcolor[HTML]{EFEFEF} 
$100$      & 0.0 & 0.0                 & 0.000 & 0.02               & 0.084 & 0.74  \\
$1,\!000$  & 0.0 & 0.0                 & 0.000 & 0.34               & 0.057 & 0.83  \\
\rowcolor[HTML]{EFEFEF} 
$10,\!000$ & 0.0 & 0.33                & 0.003 & 0.83               & 0.069 & 0.83  \\ \hline
\end{tabular}
\caption{Table showing the failure rate (FR) and proportion of significant pairwise comparison (SIG) identified for $\delta=0.05$ using different bounding techniques and sample sizes. The first column represents the sample size. The second, third, and fourth columns represent the results for PBP, PBP-t, and bootstrap bound methods respectively. For each sample size, $1,\!000$ experiments were conducted.}
\label{tab:boundresults}
\end{table}

To test the different bounding techniques, we estimate the failure rate of each confidence interval technique at different sample sizes. For this experiment we execute $1,\!000$ trials of the evaluation procedure using sample sizes (trials per algorithm per environment) of $10$, $30$, $100$, $1,\!000$, and $10,\!000$. 
There are a total of $11.14$ million samples per algorithm per environment. To reduce computation costs, we limit this experiment to only include Sarsa($\lambda$)-Parl2, Q($\lambda$)-Parl2, AC-Parl2, and Sarsa($\lambda$)-s. Additionally, we reduce the environment set to be the discrete environments and Mountain Car. We compute the failure rate of the confidence intervals, where a valid confidence interval will have a failure rate less than or equal to $\delta$, e.g., for $\delta=0.05$ failure rate should be less than $\le 5\%$. We report the failure rate and the proportion of statistically significant pairwise comparisons in Table \ref{tab:boundresults}.
All methods use the same data, so the results are not independent.

The PBP method has zero failures indicating it is overly conservative. 
The failure rate of PBP-t is expected to converge to zero as the number of samples increase due to the central limit theorem. 
PBP-t begins to identify significant results at a sample size of $1,\!000$, but it is only at $10,\!000$ that it can identify all pairwise differences.\footnote{Sarsa-Parl2 and Q-Parl2 have similar performance on discrete environments so we consider detecting $83\%$ of results optimal.} 
The bootstrap technique has the tightest intervals, but has a high failure rate. 

These results are stochastic and will not necessarily hold with different numbers of algorithms and environments. So, one should use caution in making claims that rely on either PBP-t or bootstrap. Nevertheless, to detect statistically significant results, we recommend running between $1,\!000$, and $10,\!000$ samples, and using the PBP-t over bootstrap.

While this number of trials seems, high it is a necessity as comparison of multiple algorithms over many environments is a challenging statistical problem with many sources of uncertainty.
Thus, one should be skeptical of results that use substantially fewer trials. 
Additionally, researchers are already conducting many trials that go unreported when tuning hyperparameters.
Since our method requires no hyperparameter tuning, researchers can instead spend the same amount of time collecting more trials that can be used to quantify uncertainty.

There are a few ways that the number of trials needed can be reduced.
The first is to think carefully about what question one should answer so that only a few algorithms and environments are required.
The second is to use active sampling techniques to determine when to stop generating samples of performance for each algorithm environment pair \citep{rowland2019alphabounds}. 
It is important to caution the reader that this process can bias the results if the sequential tests are not accounted for \citep{howard2018confseq}.

Summarizing our experiments, we make the following observations. Our experiments with complete algorithms show that there is still more work required to make standard RL algorithms work reliably on even extremely simple benchmark problems. As a result of our evaluation procedure, we were able to identify performance differences in algorithms that are not noticeable under standard evaluation procedures. The tests of the confidence intervals suggest that both PBP and PBP-t provide reliable estimates of uncertainty. These outcomes suggest that this evaluation procedure will be useful in comparing the performance of RL algorithms.

\section{Related Work}

This paper is not the first to investigate and address issues in empirically evaluating algorithms. The evaluation of algorithms has become a signficant enough topic to spawn its own field of study, known as \textit{experimental algorithmics} \citep{fleischer2002algorithmics,mcgeoch2012guide}.

In RL, there have been significant efforts to discuss and improve the evaluation of algorithms \citep{whiteson2011special}.
One common theme has been to produce shared benchmark environments, such as those in the annual reinforcement learning competitions \citep{whiteson2008competition,dimitrakakis2014competition}, the Arcade Learning Environment \citep{bellemare13arcade}, and numerous others which are to long to list here.
Recently, there has been a trend of explicit investigations into the reproducibility of reported results \citep{henderson2018deep,islam2017reproducibility,khetarpal2018reevaluate,colas2018seeds}. These efforts are in part due to the inadequate experimental practices and reporting in RL and general machine learning \citep{pineau2020reproducibility,lipton2018trends}.
Similar to these studies, this work has been motivated by the need for a more reliable evaluation procedure to compare algorithms. The primary difference in our work to these is that the knowledge required to use an algorithm gets included in the performance metric.

An important aspect of evaluation not discussed so far in this paper is competitive versus scientific testing \citep{hooker1995testing}. Competitive testing is the practice of having algorithms compete for top performance on benchmark tasks. Scientific testing is the careful experimentation of algorithms to gain insight into how an algorithm works. 
The main difference in these two approaches is that competitive testing only says \textit{which} algorithms worked well but not \textit{why}, whereas scientific testing directly investigates the what, when, how, or why better performance can be achieved. 

There are several examples of recent works using scientific testing to expand our understanding of commonly used methods. 
\citet{lyle2019distribuional} compares distributional RL approaches using different function approximation schemes showing that distributional approaches are only effective when nonlinear function approximation is used. 
\citet{tucker2018mirage} explore the sources of variance reduction in action dependent control variates showing that improvement was small or due to additional bias. \citet{witty2018measuring} and \citet{atrey2020saliency} investigate learned behaviors of an agent playing Atari 2600 games using ToyBox \cite{foley2018toybox}, a tool designed explicitly to enable carefully controlled experimentation of RL agents. 
While, at first glance the techniques developed here seems to be only compatible with competitive testing, this is only because we specified question with a competitive answer. 
The techniques developed here, particularly complete algorithm definitions, can be used to accurately evaluate the impact of various algorithmic choices. This allows for the careful experimentation to determine what components are essential to an algorithm.

\section{Conclusion}
\label{sec:conclusion}

The evaluation framework that we propose provides a principled method for evaluating RL algorithms. 
This approach facilitates fair comparisons of algorithms by removing unintentional biases common in the research setting. 
By developing a method to establish high-confidence bounds over this approach, we provide the framework necessary for reliable comparisons. 
We hope that our provided implementations will allow other researchers to easily leverage this approach to report the performances of the algorithms they create.

\section*{Acknowledgements}
The authors would like to thank Kaleigh Clary, Emma Tosch, and members of the Autonomous Learning Laboratory: Blossom Metevier, James Kostas, and Chris Nota, for discussion and feedback on various versions of this manuscript.
Additionally, we would like to thank the reviewers and meta-reviewers for their comments, which helped improved this paper.

This work was performed in part using high performance computing equipment obtained under a grant from the Collaborative R\&D Fund managed by the Massachusetts Technology Collaborative.
This work was supported in part by a gift from Adobe.
This work was supported in part by the Center for Intelligent Information Retrieval. Any opinions, findings and conclusions or recommendations expressed in this material are those of the authors and do not necessarily reflect those of the sponsor.
Research reported in this paper was sponsored in part by the CCDC Army Research Laboratory under Cooperative Agreement W911NF-17-2-0196 (ARL IoBT CRA). The views and conclusions contained in this document are those of the authors and should not be interpreted as representing the official policies, either expressed or implied, of the Army Research Laboratory or the U.S. Government. The U.S. Government is authorized to reproduce and distribute reprints for Government purposes notwithstanding any copyright notation herein.
%


\bibliography{main}
\bibliographystyle{icml2020}

\clearpage
\appendix
\onecolumn{
\section*{Appendix}
\input{appendix.tex}

}

\end{document}

%% file: appendix.tex
\input{appendix_main.tex}

\input{appendix_performance.tex}

%% file: appendix_main.tex
\section{Other Normalization Methods}
\label{app:normalization}

A simple normalization technique is to map scores on an environment $j$ that are in the range $[a_j,b_j]$ to $[0,1]$, i.e., $g(x,j) \coloneqq (x-a_j) / (b_j - a_j)$ \citep{bellemare13arcade}. However, this can result in normalized performance measures that cluster in different regions of $[0,1]$ for each environment. For example, consider one environment where a the minimum is $-100$, the maximum is $10$ and a uniform random policy can score around $0$. Similarly consider a second environment where the minimum score is $10$, maximum is $1,\!000$, and random gets around $20$. On the first environment, algorithms will tend to have a normalized performance near $1$ and in the second case most algorithms will have a normalized performance near $0$. So in the second environment algorithms will likely appear worse than algorithms in the first regardless of how close to optimal they are. This means the normalized performances are not really comparable.

A different version of this approach uses the minimum and maximum mean performance of each algorithm \citep{bellemare13arcade,balduzzi2018eval}. Let $\hat \mu_{i,j}$ be the sample mean of $X_{i,j}$. Then this normalization method uses the following function, $\bar g(i,j) \coloneqq (\hat \mu_{i,j} - \min_{i'} \hat \mu_{i',j}) / (\max_{i'} \hat \mu_{i',j} - \min_{i'} \hat \mu_{i',j})$. This sets the best algorithm's performance on each environment to $1$ and the worst to $0$, spreading the range of values out over the whole interval $[0,1]$. This normalization technique does not correct for nonlinear scaling of performance. As a result algorithms could be near $0$ or $1$ if there is an outlier algorithm that does very well or poorly. For example, one could introduce a terrible algorithm that just chooses one action the whole time. This makes the environment seem easier as all scores would be near $1$ except for this bad algorithm. We would like the evaluation procedure to be robust to the addition of poor algorithms.

An alternative normalization technique proposed by \citet{whiteson2011eval} uses the probability that one algorithm outperforms another on an environment, $j$, i.e., $\Pr(X_{i,j} \ge X_{k,j})$. This technique is intuitive and straight forward to estimate but neglects the difference in score magnitudes. For example, consider that algorithm $i$ always scores a $1.0$ and algorithm $k$ always scores $0.99$, the probability that $i$ is better than $k$ is $1$, but the difference between them is small, and the normalized score of $1.0$ neglects this difference.

\section{$\alpha$-Rank and our Implementation}
\label{app:alpharank}

The $\alpha$-Rank procedure finds a solution to a game by computing the stationary distribution of strategy profiles when each player is allowed to change their strategy. This is done by constructing a directed graph where nodes are pure strategies and edges have weights corresponding to the probability that one of the players switches strategies. This graph can be represented by a Markov matrix, $C \in [0,1]^{|\mathcal{S}| \times |\mathcal{S}|}$. The entry $C_{s,s'}$ corresponds to a probability of switching from a strategy $s$ to $s'$. Only one player is allowed to change strategies at a time, so the possible transitions for a strategy $s=(i,(j,k))$, are any strategies $s'=(i',(j,k))$ or $s'=(i,(j',k'))$ for all $i',k' \in \mathcal{A}$ and $j' \in \mathcal{M}$.

The typical $\alpha$-Rank procedure uses transition probabilities, $C_{s,s'}$, that are based on a logistic transformation of the payoff difference $u_l(s') - u_l(s)$. These differences are scaled by a parameter $\alpha$ and as $\alpha$ approaches $\infty$, the transition matrix approximates the Markov Conley chain (MCC), which is the motivation for using $\alpha$-Rank as a solution concept for games. See the work of \citet{omidshafiei2019alpharank} for more detailed information. 
The entries of the matrix for valid transitions are:
\begin{equation}
    C_{s,s'} =
    \begin{cases}
      \eta \frac{1-\exp{\left (-\alpha (u_l(s')-u_l(s)) \right )}}{1-\exp{\left (-\alpha n(u_l(s')-u_l(s)) \right )}} & \text{if $u_l(s') \!= u_l(s)$}\\
      \frac{\eta}{n} & \text{if $u_l(s') = u_l(s)$}\\
    \end{cases}
    \text{ and }  C_{s,s} = 1 - \sum_{s' \ne s}C_{s,s'},
\end{equation}
where $\eta = (\sum_{k} |\mathcal{S}^k|-1)^{-1}$, $l$ represents the player who switched from strategy $s$ to $s'$, $n$ is the population constant (we set it to $50$ following the prior work).
The equilibrium over strategies is then given by the stationary distribution $d$ of the Markov chain induced by $C$, i.e., $d$ is a distribution such that $d = d \tilde C$. The equilibrium solution $p^*, q^*$ are then the sum of probabilities in $d$ for each strategy, i.e., $p^*_{s_1} = \sum_{s_2 \in \mathcal{S}_2} d_{s_1,s_2}$ and $q^*_{s_2} = \sum_{s_1 \in \mathcal{S}_1} d_{s_1,s_2}$.
The aggregate performance can then be computed using $q^*$ as in \eqref{eq:aggforalgi}.

Theoretically, the hyperparameter $\alpha$ could be chosen arbitrarily high and the matrix $C$ would still be irreducible, i.e., for all $s \in \mathcal{S}$, $d_{s} > 0$ and $d$ is unique. However, due to numerical precision issues, a high value of $\alpha$ sets transition probabilities to zero for some dominated strategies, i.e., $u_l(s') < u_l(s)$, which can result in a matrix that is reducible. The suggested method to chose $\alpha$ is to tune it on a logarithmic scaled to find the highest value such that the transition matrix, $C$, is still irreducible \citep{omidshafiei2019alpharank}. 

This strategy works when the payoffs are known, but when they represent empirical samples of performance, then the value of $\alpha$ chosen will depend on the empirical payoff functions. Setting $\alpha$ based solely on the empirical payoffs could introduce bias to the matrix based on that sample. So we need a different solution without a data dependent hyperparameter.

In the MCC graph construction, all edges leading to strategies with strictly greater payoffs have the same positive weight. All edges that lead to strategies with the same payoff have he same weight but less than that of the strictly greater payoff. There are no transitions to strategies with worse payoffs. 
As $\alpha \to \infty$ the transitions probabilities quickly saturate to $\eta$ if $u_l(s') > u_l(s)$ and $0$ if $u_l(s') < u_l(s)$. So we construct the transition matrix, $C$, differently using the saturation values to set the transition probabilities $C$ is close to the MCC construction. 
The entries for $C$ that represent valid transitions in the graph are:
\begin{equation}
\label{eq:cmat}
    C_{s,s'} \coloneqq
    \begin{cases}
      \eta & \text{if $u_l(s') > u_l(s)$}\\
      \frac{\eta}{m} & \text{if $u_l(s') = u_l(s)$}\\
      0 & \text{otherwise}
    \end{cases}
    \text{ }  C_{s,s} \coloneqq 1 - \sum_{s' \ne s}C_{s,s'}.
\end{equation}
However, this often makes the transition matrix reducible, i.e., the stationary distribution might have mass on only one strategy. 
To ensure $C$ is irreducible we follow the damping approach used in PageRank \citep{page1999pagerank}, i.e., $\tilde C = \gamma C + (1-\gamma) (1/|\mathcal{S}|)$, where $\gamma \in (0,1)$ is a hyperparameter and $1-\gamma$ represents the probability of randomly switching to any strategy in $\mathcal{S}$. During a Monte-Carlo simulation of transitions through $\tilde C$, states will transition from $s$ to $s'$ according to $C$, but with probability $1-\gamma$ the transition ignores $C$ and switches to some strategy $s' \in \mathcal{S}$ chosen uniformly at random.

For $\gamma=1$ the matrix is unchanged and represents the MCC solution, but is reducible. For $\gamma$ near one, the stationary distribution will be similar to the solution given by the MCC solution with high weight placed on dominate strategies and small weight given to weak ones. As $\gamma \to 0$ the stationary distribution becomes more uniform as it is only considering shorter sequences of transitions before a random switch occurs. 
This method differs from the infinite-$\alpha$ approach presented by \citet{rowland2019alphabounds}, but in the limit as $\gamma \to 1$ and $\alpha \to \infty$, the solutions have small differences. The approach using $\gamma$ has a benefit in that there is no data dependent hyperparameter and it has a simple interpretation.

We chose to set $\gamma = (|\mathcal{S}|-1) / (|\mathcal{S}|)$ so that the expected number of transitions to occur before a random jump is $|\mathcal{S}|$. This allows for propagation of transition probabilities to cover every strategy combination. We could have chosen to set $\gamma$ near one, e.g., $\gamma = 1-10^{-8}$, but this would make the computation of the confidence intervals take longer. This is because optimizing the $C$ within confidence intervals $[C^-, C^+]$ (defined in the next section) is equivalent finding the optimal value function of a Markov decision process (MDP) with a discount parameter of $\gamma$. See the work of \citet{kerchove2007maxpage,fercoq2013pagerankopt,csaji2014pagerankopt} for more information on this connection. Solving and MDP with a discount $\gamma$ near $1.0$ causes the optimization process of value iteration and policy iteration to converge slower than if $\gamma$ is small. So we chose $\gamma$ such that it could still find solutions near the MCC solution, but remain computationally efficient.

Using this new definition of $C$ we use the following alternative but equivalent method to compute the aggregate performance more efficiently \citep{fercoq2013pagerankopt}:
\begin{align}
\label{eq:aggalgivalue}
    y_i = \frac{1 - \gamma}{|\mathcal{S}|} \sum_{s \in \mathcal{S}} v(s) \quad
    v = (I - \gamma C)^{-1} R_i,
\end{align}
where $R_i \in \mathbb{R}^{|\mathcal{S}|}$ is a vector with entries $R_i(s) \coloneqq \mathbf{E}[F_{X_{k,j}}(X_{i,j})]$ with $s=(s_1,s_2)$ and $s_2 = (j,k)$. Notice that $s_1$ is ignored because $i$ is already specified by $R_i$.

\section{Confidence Intervals on the Aggregate Performance}
\label{app:confagg}

\begin{table}
\centering
\begin{tabular}{|r|l|} 
\hline 
\multicolumn{2}{|c|}{\textbf{Symbol List}} \\
\hline
\multicolumn{1}{|c|}{\textbf{Symbol}}   & \multicolumn{1}{c|}{\textbf{Description}}                                                                                                              \\ 
\hline
\rowcolor[HTML]{EFEFEF} 
$\mathcal{A}$                             & set of algorithms in the evaluation                                                                                                                   \\
$\mathcal{M}$                             & set of environments in the evaluation                                                                                                                  \\
\rowcolor[HTML]{EFEFEF} 
$X_{i,j}$                                 & random variable representing performance of algorithm $i$ on environment $j$                                                                           \\
$T_{i,j}$                                 & number of samples of performance for algorithm $i$ on environment $j$                                                                                  \\
\rowcolor[HTML]{EFEFEF} 
$x_{i,j,t}$                               & the~$t^\text{th}$sample of performance of algorithm~ $i$ on environment $j$ and sorted such that $x_{i,j,t-1} \le x_{i,j,t}$                           \\
$D$                                       & data set containing all samples of performance for each algorithm on each environment                                                                  \\
\rowcolor[HTML]{EFEFEF} 
$y \in \mathbb{R}^{|\mathcal{A}|}$        & $y_i$ is the aggregate performance for each algorithm $i$ \\
$Y^-,Y^+ \in \mathbb{R}^{|\mathcal{A}|}$  & lower and upper confidence intervals on $y$ computed using $D$                           \\
\rowcolor[HTML]{EFEFEF} 
$\delta \in (0,0.5]$ & confidence level for the aggregate performance                                                                \\
$F_{X_{i,j}}$                             & cumulative distribution function (CDF) of$X_{i,j}$ and is also used for normalization                                                               \\
\rowcolor[HTML]{EFEFEF} 
$\hat F_{X_{i,j}}$                        & empirical cumulative distribution function constructed using samples $x_{i,j,\cdot}$ \\
$F^-_{X_{i,j}}$, $F^+_{X_{i,j}}$          & lower and upper confidence intervals on $F_{X_{i,j}}$ computed using $D$                                                                               \\
\rowcolor[HTML]{EFEFEF} 
$z_{i,j,k}$                               & performance of algorithm $i$, i.e., $z_{i,j,k} = \mathbf{E}[F_{X_{k,j}}(X_{i,j})]$ \\
$Z^-_{i,j,k},Z^+_{i,j,k}$                 & lower and upper confidence intervals on $z_{i,j,k}$ computed using $D$. \\
\rowcolor[HTML]{EFEFEF} 
$s_1 \in \mathcal{S}_1$                   & strategy for player $p$ where $\mathcal{S}_1 = \mathcal{A}$ and $s_1$ is often denoted using $i $                                                      \\
$s_2 \in \mathcal{S}_2$                   & strategy for player $q$ where $\mathcal{S}_2 = \mathcal{M} \times \mathcal{A}$ and $s_2$ is often denoted using $(j,k)$                                \\
\rowcolor[HTML]{EFEFEF} 
$s \in \mathcal{S}$                        & joint strategy where $\mathcal{S} = \mathcal{S}_1 \times \mathcal{S}_2,$ and $s=(s_1,s_2)$ is often denoted as $(i,(j,k))$      \\
$p \in \Delta(\mathcal{S}_1)$             & strategy for player $p$ represented as a distribution over $\mathcal{S}_1$                                                                             \\
\rowcolor[HTML]{EFEFEF} 
$q \in \Delta(\mathcal{S}_2)$             & strategy for player $q$ represented as a distribution over $\mathcal{S}_2$                                                                             \\
$u_p(s)$                                  & payoff for player $p$ when $s$ is played, i.e.,$ u_p(s) = \mathbf{E}[F_{X_{k,j}}(X_{i,j})]$                                                            \\
\rowcolor[HTML]{EFEFEF} 
$u_q(s)$                                  & payoff for player $q$ when $s$ is played, i.e., $u_q(s) = - u_p(s)$                                                                                    \\
$u_l^-(s), u_l^+(s)$                      & confidence intervals on $u_l(s)$ for player $l \in \{p,q\}$ computed using D                                                                           \\
\hline
\end{tabular}
\caption{List of symbols used to create confidence intervals on the aggregate performance.}
\label{tbl:symbols}
\end{table}

In this section, we detail the PBP procedure for computing confidence intervals $Y^-$ and $Y^+$ on the aggregate performance $y$ and prove that they hold with high probability.  
That is, we show that for any confidence level $\delta \in (0, 0.5]$;
\begin{equation}
    \Pr(\forall i \in \mathcal{A}, \ y_i \in [Y^-_i, Y^+_i]) \ge 1 - \delta.
\end{equation}

We will first describe the PBP procedure to compute confidence intervals and then prove that they valid. 
A list of the symbols used in the construction of confidence intervals and their description are provided in Table \ref{tbl:symbols} to refresh the reader. 
The steps to compute the confidence intervals are outlined in Algorithm \ref{alg:pbp}.

Recall that the aggregate performance for an algorithm $i$ is 
\begin{equation}
    y_i \coloneqq \sum_{j=1}^{|\mathcal{M}|}\sum_{k=1}^{|\mathcal{A}|} q^*_{j,k} \mathbf{E}[F_{X_{k,j}}(X_{i,j})],
\end{equation}
where $q^*$ is the equilibrium solution to the game specified in Section \ref{sec:summarization}. 
To compute valid confidence intervals $Y^-,Y^+$ on $y$ using a dataset $D$, the uncertainty of $q^*$ and mean normalized performance $z_{i,j,k} = \mathbf{E}[F_{X_{k,j}}(X_{i,j})]$. PBP accomplishes this by three primary steps. The first step is to compute confidence intervals $Z^-_{i,j,k},Z^+_{i,j,k}$ on $z_{i,j,k}$ such that
\begin{equation}
    \Pr \left (\forall (i,j,k) \in \mathcal{A} \times \mathcal{M} \times \mathcal{A},\  z_{i,j,k} \in [Z^-_{i,j,k}, Z^+_{i,j,k}] \right ) \ge 1-\delta. 
\end{equation}
The second step is to compute the uncertainty set $\mathcal{Q}$ containing all possible $q^*$ that are compatible with $Z^-$ and $Z^+$. 
The last step is to compute the smallest and largest possible aggregate performances for each algorithm over these sets, i.e.,
\begin{align}
    Y^-_i = \min_{q \in \mathcal{Q}} \sum_{j=1}^{|\mathcal{M}|}\sum_{k=1}^{|\mathcal{A}|} q_{j,k} Z^-_{i,j,k} \text{ and }\quad
    Y^+_i =  \max_{q \in \mathcal{Q}} \sum_{j=1}^{|\mathcal{M}|}\sum_{k=1}^{|\mathcal{A}|} q_{j,k} Z^+_{i,j,k}. \\
\end{align}
PBP follows this process, except in the last two steps $\mathcal{Q}$ is never explicitly constructed to improve computational efficiency. 
Intuitively, the procedure provides valid confidence intervals because all values to compute the aggregate performance depend on the normalized performance. So by guaranteeing with probability at least $1-\delta$ that the true mean normalized performances will be between $Z^-$ and $Z^+$, then so long as the the upper (lower) confidence interval computed is at least as large (small) as the maximum (minimum) of the aggregate score for any setting of $z \in [Z^-, Z^+]$, the confidence intervals will be valid.

We break the rest of this section into two subsections. The first subsection discusses constructing the confidence intervals on the mean normalized performance and proving their validity. The second subsection describes how to construct the confidence intervals on the aggregate performance proves their validity. 

\begin{algorithm}[tb]
   \caption{Performance Bound Propagation (PBP)}
   \label{alg:pbp}
\begin{algorithmic}[1]
\STATE {\bfseries Input:} dataset $D$ containing samples of performance and a confidence level $\delta \in (0,0.5])$
\STATE {\bfseries Output:} $Y^-$, $Y^+$ confidence intervals on the aggregate performance
\\
\hrulefill 
\STATE $\delta^\prime \gets \delta / (|\mathcal{A}||\mathcal{M}|)$;
\STATE $\texttt{sort\_ascending}(\{x_{i,j,t}\}_{t=1}^{T_{i,j}})$;
\STATE // Compute confidence intervals for the CDFs \label{alg:lines:dkwstart}
\FOR{$i,j \in \mathcal{A} \times \mathcal{M}$}
    \STATE $F_{X_{i,j}}^-, F_{X_{i,j}}^+ \gets  \texttt{dkw\_bound}(\{x_{i,j,t}\}_{t=1}^{T_{i,j}}, \delta^\prime)$; // computation shown in \eqref{eq:dkw} \label{alg:lines:dkwstop}
\ENDFOR 
\STATE // Compute confidence intervals on the mean normalized performance
\FOR{$i,j,k \in \mathcal{A} \times \mathcal{M} \times \mathcal{A}$}
    \STATE $Z^-_{i,j,k} \gets F^-_{X_{k,j}}(x_{i,j,T}) -  \sum_{t=0}^{T-1} \left [ F^-_{X_{k,j}}(x_{i,j,t+1})-F^-_{X_{k,j}}(x_{i,j,t}) \right ] F^+_{X_{i,j}}(x_{i,j,t})$; \\
    \STATE $Z^+_{i,j,k} \gets F^+_{X_{k,j}}(x_{i,j,T+1}) -  \sum_{t=1}^{T} \left [ F^+_{X_{k,j}}(x_{i,j,t+1})-F^+_{X_{k,j}}(x_{i,j,t}) \right ] F^-_{X_{i,j}}(x_{i,j,t})$;
\ENDFOR
\STATE // Construct game quantities
\STATE $\mathcal{S} = \mathcal{A} \times (\mathcal{M} \times \mathcal{A})$; \ strategy profile set
\STATE $\gamma \gets \frac{|\mathcal{S}|-1}{|\mathcal{S}|}$
\STATE $C^-, C^+ \gets \texttt{bound\_markov\_matrix}(Z^-, Z^+)$ as defined in \eqref{eq:cbounds}.
\STATE // Optimize aggregate performance over all possible $C \in [C^-,C^+]$
\FOR{$i \in \mathcal{A}$}
    \STATE $v \gets \texttt{find\_optimal\_valuefunction}(C^-, C^+, R_i = -Z^-_{i,\cdot,\cdot})$; // solve \eqref{eq:lpbound}
    \STATE $Y^-_i \gets \frac{(1-\gamma)}{|\mathcal{S}|} \sum_{s \in \mathcal{S}}|v(s)|$
    \STATE $v \gets \texttt{find\_optimal\_valuefunction}(C^-, C^+, R_i = Z^+_{i,\cdot,\cdot})$; // solve \eqref{eq:lpbound} 
    \STATE $Y^+_i \gets \frac{(1-\gamma)}{|\mathcal{S}|} \sum_{s \in \mathcal{S}}|v(s)|$
\ENDFOR
\end{algorithmic}
\end{algorithm}

\subsection{Confidence intervals on the normalized performance}
The normalized performance $z_{i,j,k} = \mathbf{E}[F_{X_{k,j}}(X_{i,j})]$ has two unknowns, $F_{X_{k,j}}$ and the distribution of $X_{i,j}$. 
To compute confidence intervals on $z_{i,j,k}$ for all $i,j,k$, confidence intervals are needed on the output on all distribution functions $F_{X_{i,j}}$. 
The confidence intervals on the distributions can then be combined to get confidence intervals on $z_{i,j,k}$.

To compute confidence intervals on $F_{X_{i,j}}$ we assume that $X_{i,j}$ is bounded on the interval $[a_j, b_j]$ for all $i \in \mathcal{A}$ and $j \in \mathcal{M}$. Let $\hat F_{X_{i,j}}$ be the empirical CDF with
\begin{equation}
    \hat F_{X_{i,j}}(x) \coloneqq \frac{1}{T_{i,j}} \sum_{t=1}^{T_{i,j}} \mathbf{1}_{x_{i,j,t} \le x},
\end{equation}
where $T_{i,j}$ is the number of samples of $X_{i,j}$, $x_{i,j,t}$ is the $t^\text{th}$ sample of $X_{i,j}$, and $\mathbf{1}_{A} = 1$ if event $A$ is true and $0$ otherwise.
Using the Dvoretzky--Kiefer--Wolfowitz (DKW) inequality \citep{dvoretzky1956bound} with tight constants \citep{massart1990bound}, we define $F^-_{X_{i,j}}$ and $F^+_{X_{i,j}}$ to be the lower and upper confidence intervals on $F_{X_{i,j}}$, i.e.,
\begin{equation}
\label{eq:dkw}
    \begin{aligned}
    F^+_{X_{i,j}}(x) \coloneqq &\ 
    \begin{cases}
      1 & \text{if $x \ge b$}\\
      \min(1.0, \hat F_{X_{i,j}}(x) + \epsilon)  & \text{if $a \le x < b$}\\
      0 & \text{if $x < a$} \\
    \end{cases} \\
    F^-_{X_{i,j}}(x) \coloneqq &\ 
    \begin{cases}
      1 & \text{if $x \ge b$}\\
      \max(0.0, \hat F_{X_{i,j}}(x) - \epsilon)  & \text{if $a \le x < b$}\\
      0 & \text{if $x < a$} \\
    \end{cases}
    \end{aligned}
    \quad \text{ and } \epsilon = \sqrt{\frac{\ln \frac{2}{\delta'}}{2T_{i,j}}},
\end{equation}
where $\delta' \in (0,0.5]$ is a confidence level and we use $\delta' = \delta / (|\mathcal{A}||\mathcal{M}|)$. 
By the DKW inequality with tight constants the following property is known:
\begin{prop}[DKW with tight constants confidence intervals]
\label{prop:dkw}
\begin{equation}
    \Pr \left ( \forall x \in \mathbb{R}, \ \ F_{X_{i,j}}(x) \in [F_{X_{i,j}}^-(x), F_{X_{i,j}}^+(x)] \right ) \ge 1-\delta'.
\end{equation}
\end{prop}
\begin{proof}
See the works of \citet{dvoretzky1956bound} and \citet{massart1990bound}.
\end{proof}

Further, by the union bound we have that 
\begin{align}
\label{eq:unioncdf}
    \Pr \left (\forall i \in \mathcal{A}, \forall j \in \mathcal{M}, \forall x \in \mathbb{R}, \ \ F_{X_{i,j}}(x) \in [F_{X_{i,j}}^-(x), F_{X_{i,j}}^+(x)] \right ) \ge 1-\delta.
\end{align}

To construct confidence intervals on the mean normalized performance, we will use Anderson's inequality \citep{anderson1969mean}. Let $X$ be a bounded random variable on $[a,b]$, with sorted samples $x_1 \le x_2 \le \dots \le x_T$, $x_0=a$, and $x_{T+1}=b$. Let $g \colon \mathbb{R} \to \mathbb{R}$ be a monotonically increasing function. Anderson's inequality specifies for a confidence level $\delta \in (0, 0.5]$ the following high confidence bounds on $\mathbf{E}[g(X)]$: 
\begin{equation}
\begin{aligned}
    \mathbf{E}[g(X)] \ge & \ g(x_{T}) &-  \sum_{t=0}^{T-1} \left [ g(x_{t+1})-g(x_{t}) \right ] F^+_{X}(x_{t}) \\
    \mathbf{E}[g(X)] \le & \ g(x_{T+1}) &-  \sum_{t=1}^{T} \left [ g(x_{t+1})-g(x_{t}) \right ] F^-_{X}(x_{t}), \\
\end{aligned}
\end{equation}
where $F^{+/-}_X$ uses the DKW inequality with tight constants and as defined in \eqref{eq:dkw}.

Anderson's inequality can be used to bound the mean normalized performance since $F_{X_{k,j}}$ is a monotonically increasing function and $\delta \in (0,0.5]$. Since $F_{X_{k,j}}$ is unknown, we replace $g$ in Anderson's inequality with $F^-_{X_{k,j}}$ for the lower bound and $F^+_{X_{k,j}}$ for the upper bound. This gives the following confidence intervals for $z_{i,j,k}$:
\begin{equation}
\begin{aligned}
    Z^-_{i,j,k} = & \ F^-_{X_{k,j}}(x_{i,j,T}) &-  \sum_{t=0}^{T-1} \left [ F^-_{X_{k,j}}(x_{i,j,t+1})-F^-_{X_{k,j}}(x_{i,j,t}) \right ] F^+_{X_{i,j}}(x_{i,j,t}) \\
    Z^+_{i,j,k} = & \ F^+_{X_{k,j}}(x_{i,j,T+1}) &-  \sum_{t=1}^{T} \left [ F^+_{X_{k,j}}(x_{i,j,t+1})-F^+_{X_{k,j}}(x_{i,j,t}) \right ] F^-_{X_{i,j}}(x_{i,j,t}), \label{eq:normbounds} \\
\end{aligned}
\end{equation}
where $T = T_{i,j}$, $x_{i,j,0} = a_j$, and $x_{i,j,T+1} = b_j$.
We now prove the following lemma:
\begin{lemma}
\label{lem:meannormalized}
If $Z^-$ and $Z^+$ are computed by \eqref{eq:normbounds}, then: 
\begin{equation}
\Pr \left (\forall i,k \in \mathcal{A}, \forall j \in \mathcal{M}, \ \ z_{i,j,k} \in [Z^-_{i,j,k}, Z^+_{i,j,k}] \right ) \ge 1-\delta.
\end{equation}
\end{lemma}
\begin{proof}
By Anderson's inequality we know that $Z^+_{i,j,k}$ is an high confidence upper bound on $\mathbf{E}[F^+_{X_{k,j}}(X_{i,j})]$ and similarly $Z^-_{i,j,k}$ is a high confidence lower bound on $\mathbf{E}[F^+_{X_{k,j}}(X_{i,j})]$, i.e., 

\begin{align}
    &\ \Pr \left (\mathbf{E}[F^-_{X_{k,j}}(X_{i,j})] \ge Z^-_{i,j,k} \right ) \ge 1-\delta'/2 \\
    &\ \Pr \left (\mathbf{E}[F^+_{X_{k,j}}(X_{i,j})] \le Z^+_{i,j,k} \right ) \ge 1-\delta'/2.
\end{align}

By Property \ref{prop:dkw} we know that $\Pr \left ( \forall x \in \mathbb{R}, \ \ F_{X_{k,j}}(x) \in [F_{X_{k,j}}^-(x), F_{X_{k,j}}^+(x)] \right ) \ge 1-\delta'$, thus 
\begin{equation}
\Pr \left (\forall i,k \in \mathcal{A}, \forall j \in \mathcal{M}, \ \ z_{i,j,k} \in [Z^-_{i,j,k}, Z^+_{i,j,k}] \right ) \ge 1-2\delta',
\end{equation}
where $2\delta'$ comes from combining the failure rates of confidence intervals on the CDFs $F_{X_{i,j}}$ and $F_{X_{k,j}}$. The confidence intervals on the mean normalized performances can only fail if the confidence intervals on CDFs fail. As stated in \eqref{eq:unioncdf}, all confidence intervals on the CDFs contain the true CDFs with probability at least $1-\delta$. Thus, all mean normalized performances hold with probability at least $1 - \delta$.
\end{proof}

The confidence intervals given by $Z^-$ and $Z^+$ are guaranteed to hold for $T \ge 1$, and $\delta \in (0,0.5]$, but are often conservative requiring a large number samples to identify a statistically meaningful result. 
So we empirically test alternatives that have either stricter assumptions or weaker theoretical justification.

\subsection{Confidence intervals on the aggregate performance}
In this section we will provide details on how to compute the confidence intervals on the aggregate performance using the confidence intervals on the mean normalized performance and then prove that they hold with high confidence. To construct confidence intervals and prove their validity we will make the following steps. First, we show that for a fixed weighting $q$ that valid confidence intervals can be computed directly by using interval arithmetic. Second, we describe how to characterize the uncertainty of the game. Third, we make a connection between aggregate performance using the equilibrium solution $q^*$ and the optimal average reward for a Markov decision process. Lastly we describe an optimization procedure for computing the optimal average reward, which corresponds to finding the lower and upper confidence intervals.

Before discussing how to bound the aggregate performance using the game theoretic solution, consider the case when weights $q$ can be any probability distribution over algorithms and environments chosen before the experiment begins. 
Let weights $q \in [0,1]^{|\mathcal{A}| \times |\mathcal{M}|}$, such that $\sum_{j \in \mathcal{M}} \sum_{k \in \mathcal{A}} q_{j,k}=1$.
Let $\tilde y_i = \sum_{j \in \mathcal{M}} \sum_{k \in \mathcal{A}} q_{j,k} z_{i,j,k}$ be the aggregate performance for an algorithm $i \in \mathcal{A}$.
The corresponding confidence intervals for $\tilde y_i$ are $\tilde Y^-_i = \sum_{j \in \mathcal{M}} \sum_{k \in \mathcal{A}} q_{j,k} Z^-_{i,j,k}$ and $\tilde Y^+_i = \sum_{j \in \mathcal{M}} \sum_{k \in \mathcal{A}} q_{j,k} Z^+_{i,j,k}$.
\begin{lemma}
\label{lem:aggsum}
If weights $q$ are independent of the data $D$, then: 
\begin{equation}
 \Pr \left (\forall i \in \mathcal{A}, \  \tilde y_i \in [\tilde Y_i^-, \tilde Y_i^+] \right ) \ge 1 - \delta
 \end{equation}
\end{lemma}

\begin{proof}
Applying the result of Lemma \ref{lem:meannormalized}, all confidence intervals produced by $Z^-_{i,j,k}$ and $Z^+_{i,j,k}$ contain $z_{i,j,k}$ with probability $1-\delta$. So interval arithmetic can be used to compute confidence intervals on the aggregate performance without changing the probability of failure, i.e., $\tilde Y^-_i = \sum_{j \in \mathcal{M}} \sum_{k \in \mathcal{A}} q_{j,k} Z^-_{i,j,k}$ and $\tilde Y^+_i = \sum_{j \in \mathcal{M}} \sum_{k \in \mathcal{A}} q_{j,k} Z^+_{i,j,k}$ for each algorithm $i \in \mathcal{A}$.
With these intervals then 
\begin{align}
\Pr \left (\forall i \in \mathcal{A}, \  \tilde y_i \in [\tilde Y_i^-, \tilde Y_i^+] \right ) \ge &\ 1 - \Pr \left (\bigcup_{i,k \in \mathcal{A},j \in \mathcal{M}} z_{i,j,k} \notin  [Z^-_{i,j,k}, Z^+_{i,j,k}] \right ) \\
\ge &\ 1 - \delta.
\end{align} 
\end{proof}
This method of aggregating performance highlights how the uncertainty of normalized scores propagates to the confidence intervals of the aggregate performance for a fixed weighting. Next we will show how to compute confidence intervals on $y$ when using the dynamic weighting produce by the equilibrium solution to the two player game.

Instead of considering all possible equilibrium solutions $q^*$, recall that $y_i$ depends on the Markov matrix $C$ as shown in \eqref{eq:aggalgivalue}. To construct confidence intervals on $y_i$ the uncertainty in creating the matrix $C$ as defined in \eqref{eq:cmat} needs to be considered \citep{rowland2019alphabounds}. 
The definition of $C$ assumes certainty of payoffs of each strategy, but the empirical payoffs have uncertainty corresponding to $Z^-$ and $Z^+$, i.e., for $s=(i,(j,k))$, $u^-_p(s)=Z^-_{i,j,k}$, $u^+_p(s)=Z^+_{i,j,k}$, $u^-_q(s)=-Z^+_{i,j,k}$, $u^+_q(s)=-Z^-{i,j,k}$. As a result, when the payoff confidence intervals overlap for two strategies $s$ and $s'$ this creates uncertainty in $C$. We define $C^-,C^+ \in [0,1]^{|\mathcal{S}| \times |\mathcal{S}|}$ as the lower and upper confidence intervals on $C$ with entries 
\begin{equation}
\label{eq:cbounds}
    \begin{aligned}
    C^-_{s,s'}, C^+_{s,s'} \coloneqq & 
    \begin{cases}
      (\eta,\eta) & \text{if $u^-_l(s') > u^+_l(s)$}\\
      (0,0) & \text{if $u^-_l(s) > u^+_l(s')$}\\
      (\frac{\eta}{n},\frac{\eta}{n}) & \text{if $u^{+/-}_l(s) = u^{+/-}_l(s')$}\\
      (0,\eta) & \text{otherwise}
    \end{cases} \quad \forall s' \in \mathcal{S} \setminus \{s\} \\
        C^-_{s,s}, C^+_{s,s} \coloneqq &\ \left (1 - \sum_{s' \ne s}C^+_{s,s'},\  1 - \sum_{s' \ne s}C^-_{s,s'} \right )
    \end{aligned}
\end{equation}

To get the bounds on the aggregate performance over all possible $C$ in these intervals the uncertainty of the stationary distribution on $\tilde C = \gamma C + (1-\gamma)(1/|\mathcal{S}|)$ has to be consider. 
This can be accomplished by first computing the minimum and maximum values of the stationary distribution for each strategy \citep{rowland2019alphabounds} and then finding the minimum and maximum aggregate performances for all possible stationary distributions in this limits.
However, individually computing these two quantities leads to looser bounds than directly estimating the minimum and maximum aggregate performance over all possible $C$ because it ignores the correlations in the confidence intervals of the stationary distribution.
To compute the minimum and maximum aggregate performance over all possible $C$, we need to make a connection between the average performance using the stationary distribution of $\tilde C$ and the average performance before termination on $C$.

Let $d$ be a stationary distribution over strategy profiles $\mathcal{S}$ induce by the Markov matrix $\tilde C = \gamma C + (1-\gamma)(1/|\mathcal{S}|)$. 
Let $q$ be the distribution of strategies for player $q$ contained in $d$. 
\begin{lemma}
\label{lem:disttovalue}
The following are equivalent
\begin{align}
    y_i = &\ \sum_{(j,k) \in \mathcal{S}_2} q_{j,k}z_{i,j,k} \\
    y_i = &\ \frac{1-\gamma}{|\mathcal{S}|}\sum_{s \in \mathcal{S}} v(s),
\end{align}
where $v = (I - \gamma C)^{-1}R_i$, $R_i \in \mathbb{R}^{|\mathcal{S}|}$ such that $R_i(s) = z_{i,j,k}$ for $s=(\cdot, (j,k))$.
\end{lemma}
\begin{proof}
We know that $q_{j,k} = \sum_{i \in \mathcal{S}_1} d(i,(j,k))$. This implies $y_i = \sum_{(j,k) \in \mathcal{S}_2} q_{j,k}z_{i,j,k} = \sum_{s \in \mathcal{S}} d(s) R_i(s)$. Applying $8.2.12$ of \citet{puterman1994mdp}, we have the relation
\begin{align}
v + y_i = &\ R_i + \tilde C v \\
v + y_i = &\ R_i + \gamma C v + \frac{(1-\gamma)}{|S|}\mathbf{1}^\top v,
\end{align}
where here $\mathbf{1} \in \mathbb{R}^{|\mathcal{S}|}$ is a vector of all ones. Then for $v = (I-\gamma C)^{-1}R_i$, $y_i=\frac{(1-\gamma)}{|S|}\mathbf{1}^\top v$.
\end{proof}

Finally, using techniques developed by \citet{fercoq2013pagerankopt}, the lower and upper bounds of the aggregate performance $y_i$ can be bound by solving the following optimization problem:
\begin{equation}
\label{eq:lpbound}
\begin{aligned}
\min_{v} \quad & \sum_{s \in \mathcal{S}} v(s) \\
\textrm{s.t.} \quad & v(s) \ge R_i(s) + \gamma \sum_{s'\in \mathcal{S}} C_{s,s'} v(s')\\
\quad & C^- \le C \le C^+ \\
& \sum_{s' \in \mathcal{S}} C_{s,s'} = 1 \quad \forall s \in \mathcal{S},
\end{aligned}
\end{equation}
where $C$ is a free variable, $R_i(s) = -Z^-_{i,j,k}$, and $R_i(s) = Z^+_{i,j,k}$ are used to obtain the lower and upper bounds on $y_i$, respectively. Both bounds are computed as $Y_i^{+/-}=(1-\gamma)(1/|\mathcal{S}|) \sum_{s \in \mathcal{S}}|v(s)|$ using the respective solutions $v$. The absolution value of $v$ is taken to account for the negativity of $R_i = -Z^-$. We compute the solution $v$ in polynomial time using policy iteration based approach similar to \citet{fercoq2013pagerankopt}, but modify their algorithm to fit our matrix $C$. We detail our exact method in Appendix \ref{app:valueiteration}.

Now to prove the main result we will use the previous lemmas in connection with PBP. 

\pbpthm*
\begin{proof}
From Lemma \ref{lem:meannormalized} we know that $Z^-$ and $Z^+$ are valid $1-\delta$ confidence intervals on $z$. Thus, applying Lemma \ref{lem:aggsum} we know that a valid $1-\delta$ confidence intervals can be computed by a weighted sum of lower and upper bounds $Z^-$ and $Z^+$, for any joint probability distribution $q$ over environments and algorithms. Through Lemma \ref{lem:disttovalue} this is equivalent to assuming fixed Markov matrix $C$. 
The true matrix $C$ is unknown, so the minimum and maximum intervals need to be found over a set $\mathcal{C}$ that contains all transition matrices that are compatible with $Z^-$ and $Z^+$. 
Let $\mathcal{C} = \prod_{s \in \mathcal{S}} \mathcal{C}_s$, where $\mathcal{C}_s$ is the polytope of transition probabilities starting from strategy profile $s \in \mathcal{S}$, i.e., 
$$\mathcal{C}_s = \{C_{s,\cdot} \colon \forall {s' \in \mathcal{S}}, \  C_{s,s'} \in [C^-_{s,s'}, C^+_{s,s'}], \sum_{s' \in \mathcal{S}} C_{s,s'}=1 \}.$$
The minimum and maximum aggregate values $Y^-$ and $Y^+$ computed over $\mathcal{C}$ can be found in polynomial time using linear programming, value iteration, or policy iteration \citep{fercoq2013pagerankopt}.
PBP finds the minimum and maximum confidence intervals over all $C \in \mathcal{C}$, thus, they represent valid $1-\delta$ confidence intervals for each algorithm. 
\end{proof}

\section{Policy Iteration for Bounding the Aggregate Performance}
\label{app:valueiteration}

This section we detail our approach for optimizing the aggregate performance over the uncertainty of the Markov matrix $C$ to compute confidence intervals. 
Recall that the upper and lower high-confidence bounds on the aggregate performance for algorithm $i$ can be found by solving the optimization problem in \ref{eq:lpbound}.
Alternatively one can use a Dynamic Programming approach either using value iteration or policy iteration for more efficient optimization \citep{fercoq2013pagerankopt}. 
Using value iteration has better scaling to the size of the state space in the optimization problem than policy iteration. 
However, we found that in small and moderate sized problems policy iteration was sufficient and used it in our experiments.

Our method is similar to that of \citet[Algorithm 1]{fercoq2013pagerankopt}, except that we use policy iteration instead of value iteration and a modification to the dynamic programming operator since the transition probabilities in our problem do not depend on the number of out going edges.
Algorithm \ref{alg:policyiteration} shows pseudocode of the policy iteration to find the the optimal value function $v^*$ through the modification of the matrix $C \in \mathcal{C}$. 
The algorithm takes as input lower and upper bounds $C^-, C^+ \in \mathbb{R}^{|\mathcal{S}| \times | \mathcal{S}|}$ on $C$, a reward function $R \in \mathbb{R}^{|\mathcal{S}|}$, a discount factor $\gamma \in (0,1)$, a tolerance on the distance to optimal, $tol \in \mathbb{R}$, and the maximum number of iterations $max\_itrs \in \mathbb{N}^+$. We set $\gamma = (1-|\mathcal{S}|)/|\mathcal{S}|$ and $tol=1 \times 10^{-7}$, $max\_itrs = 400$ as default parameters. When finding the upper confidence interval, $y^+_i$, for algorithm $i$, $R$ is a vector such that $R(s) = Z^+_{i,j,k}$ for all $s=(\cdot,(j,k)) \in \mathcal{S}$. Similarly, when finding $y_i^-$, $R$ is such that $R(s) = -Z^-_{i,j,k}$.

\begin{algorithm}[tb]
   \caption{Policy Iteration for aggregate performance optimization}
   \label{alg:policyiteration}
\begin{algorithmic}
    \STATE {\bfseries Input:} $C^-, C^+ \in \mathbb{R}^{|\mathcal{S}|\times |\mathcal{S}|}$, $R \in \mathbb{R}^{|\mathcal{S}|}$, $\gamma \in (0,1) = (1-|\mathcal{S}|)/|\mathcal{S}|$, $tol \in \mathbb{R} = 1 \times 10^{-7}$, $max\_itrs \in \mathbb{N}^+ = 400$
    \STATE Initialize: $C \gets C^-$, $v \gets \mathbf{0} \in \mathbb{R}^{|\mathcal{S}|}$
    \STATE $changed \gets \texttt{True}$
    \STATE $iteration \gets 0$
    \WHILE{changed}
        \STATE $changed \gets \texttt{False}$
        \STATE $iteration \gets iteration + 1$
        \IF{$iteration \ge max\_itrs$}
        \STATE \texttt{break}
        \ENDIF
        \FOR{$s \in \mathcal{S}$}\STATE $C_s' \gets \texttt{update\_transition\_row}(C^-_s,C^+_s,R(s),v,\gamma)$
            \IF{$||C_s - C_s'||_\infty \ge 1 \times 10^{-8}$}
                \STATE $changed \gets \texttt{True}$
                \STATE $C_s \gets C_s'$
            \ENDIF
        \ENDFOR
        \STATE $v' \gets (I - \gamma C)^{-1}R$
        \STATE $\epsilon \gets ||v-v'||_\infty$  // max change in value function
        \STATE $\epsilon_{v^*} \gets (2 \epsilon \gamma) / (1-\gamma)$ // error bound on distance to $v^*$
        \STATE $\epsilon_{aggregate} \gets (1-\gamma) \epsilon_{v^*}$  // bound on the maximum error to the aggregate performance
        \IF{$\epsilon_{aggregate} < tol$}
            \STATE $changed \gets \texttt{False}$
        \ENDIF
        \STATE $v \gets v'$
    \ENDWHILE
    \STATE {\bfseries Return:} $(1-\gamma) \texttt{mean}(|v|)$
\end{algorithmic}
\end{algorithm}

Algorithm \ref{alg:policyiteration} has two main steps. First, greedily optimize $C$ with respect to the current value function $v$, i.e., for all $s \in \mathcal{S}$ 
\begin{equation}
    C_s = \argmax_{C_s \in \mathcal{C}_s} R(s) + \gamma C_s v.
\end{equation}
We provide pseudocode for this step in Algorithm \ref{alg:policyupdate}. 
The second step is a value function update, which we compute exactly by solving the system of linear equations 
$    v = R + \gamma C v$, 
by setting $v = (I-\gamma C)^{-1}R$.
These steps repeat until $C$ stops changing or a bound on the maximum absolute error in aggregate performance is below some threshold. 
The confidence interval on the aggregate performance is then returned as $(1-\gamma)/|\mathcal{S}| \sum_{s \in \mathcal{S}} |v(s)|$. 
Since $C$ is sparse optimization to the code can be made to drastically speed up computation when $\mathcal{S}$ is large. We make some of these modifications in our implementation.

\begin{algorithm}[tb]
   \caption{\texttt{update\_transition\_row} procedure for updating $C_s$}
   \label{alg:policyupdate}
\begin{algorithmic}
   \STATE {\bfseries Input:} $C^-_s, C^+_s \in \mathbb{R}^{|\mathcal{S}|}$, $r \in \mathbb{R}$, $v \in \mathbb{R}^{|\mathcal{S}|}$, $\gamma \in (0,1)$
   \STATE Initialize: $C_s \gets C^-_s$, $c \gets \texttt{sum}(C^-_s)$
   \STATE $w \gets r + \gamma v$
   \STATE $idxs \gets \texttt{argsort}(w, \texttt{direction}=decreasing)$
   \FOR{$i=1$ {\bfseries to} $|\mathcal{S}|$}
   \STATE $idx \gets idxs[i]$
   \STATE $\Delta c \gets \min(C^+_s-C^-_s, 1-c)$
   \STATE $C_s[idx] \gets C_s[idx] + \Delta c$
   \STATE $c \gets c + \Delta c$
   \ENDFOR
   \STATE {\bfseries Return:} $C_s$
\end{algorithmic}
\end{algorithm}

Due to small numerical errors this version of policy iteration may not keep the same policy, $C$, between successive iterations when at the optimal solution. To ensure that the procedure stops in a reasonable time and closely approximates the true solution we employ three techniques: an iteration limit, halting computation when $C$ is $\epsilon$-close to between iterations, and stopping when a bound on the distance to true solution is below a tolerance, $tol$. 
The first two approaches are straight forward an in the pseudocode. 
To bound the distance to the true solution we leverage prior work on bounding the distance of the current value function to the optimal value function. 
Consider the value function $v$ and the subsequent value function $v'$ obtained after one application of the Bellman operator. 
In our problem the Bellman operator is
\begin{equation}
    v(s) = \max_{C_s \in \mathcal{C}_s} R(s) + \gamma \sum_{s' \in \mathcal{S}}C_{s,s'} v(s').
\end{equation}
Let $\epsilon = \max_{s \in \mathcal{S}} |v(s)-v'(s)|$. 
Then the distance $\epsilon_{v^*} = \max_{s \in \mathcal{S}}|v'(s)-v^*(s)|$ of $v'$ to the optimal value function $v^*$, is bounded above by $\epsilon$ \citep{williams1993bounds}, i.e, 
$$\epsilon_{v^*} \le (2 \epsilon \gamma) / (1-\gamma).$$
We translate this this bound into a bound on the error to the confidence interval of $y_i$. 
Let $y_i$ be the confidence interval computed using $v'$ and $y_i^*$ be the confidence interval computed using $v^*$. 
Then an upper bound on the error $\epsilon_{aggregate} = |y_i^*-y_i|$ is

\begin{align}
    \epsilon_{aggregate} = &\ |y_i^* - y_i| \\
    = &\ \left | \frac{1-\gamma}{|\mathcal{S}|} \sum_{s \in \mathcal{S}} |v^*(s)| - \frac{1-\gamma}{|\mathcal{S}|} \sum_{s \in \mathcal{S}}  |v'(s)| \right | \\
    = &\ \left | \frac{1-\gamma}{|\mathcal{S}|} (||v^*||_1 - ||v'||_1) \right |  \\
    \le &\ \frac{1-\gamma}{|\mathcal{S}|} ||v^*-v'||_1  \\
    = &\ \frac{1-\gamma}{|\mathcal{S}|} |\mathcal{S}| ||v^*-v'||_\infty  \\
    = &\ (1-\gamma) ||v^*-v'||_\infty  \\
    \le &\ (1-\gamma) \max_{s \in \mathcal{S}}|v^*(s)-v'(s)|  \\
    = &\ (1-\gamma) \epsilon_{v^*}  = 2 \epsilon \gamma.
\end{align}

\section{Algorithm Definitions}
\label{app:definitions}
This section provides the complete definition for each algorithm used in the experiments. 
Each algorithm is made complete by defining a distribution from, which hyperparameters are sampled. Table \ref{tab:hyperparametertable} shows distributions or values for any hyperparameter used in this work. For the continue state space environments, all of the algorithms use Fourier basis and linear function approximation \citep{konidaris2011fourier}. Note that $U(a,b)$ indicates a uniform random variable on $[a,b)$, and $U(\{\ldots \})$ indicates that a variable is sampled uniformly at random from a set of finite values.

Note that these ranges should not be considered optimal and could easily be improved for the environments in this experiment. The ranges were chosen to reflect a lack of knowledge about what optimal settings on an environment are and to be reflective of ranges one might expect optimal hyperparameters to fall in. The manual setting of these ranges leaks information based on our own experience with the algorithms and environments. However, since the definition is completely specified on these environments any favor to one algorithm could be easily identified and test for. Furthermore, any adaptive algorithm that can adjust these parameters during learning is likely to be superior than specifying better ranges through experience in this exact setup. Still, one should not tune ranges to fit any given set of environments. 

\begin{table}
\centering
\begin{tabular}{rrcc}
\hline
\multicolumn{1}{c}{\textbf{Algorithm}} & \multicolumn{1}{c}{\textbf{Hyperparameter}} & \textbf{Discrete}                & \textbf{Continuous}                                                 \\ \hline
\rowcolor[HTML]{EFEFEF} 
All                                    & $\lambda$                                   & \multicolumn{2}{c}{\cellcolor[HTML]{EFEFEF}$U(0,1)$}                                                   \\
All                                    & $\gamma$                                    & \multicolumn{2}{c}{$\Gamma - e^{U(\ln 10^{-4},\ln 0.05)}$}                                             \\
\rowcolor[HTML]{EFEFEF} 
All                                    & Value function                              & Tabular                          & Linear with Fourier basis                                           \\
All Sarsa$(\lambda)$ and Q$(\lambda)$  & $\epsilon$                                  & \multicolumn{2}{c}{$U(0,1)$}                                                                           \\
\rowcolor[HTML]{EFEFEF} 
Sarsa$(\lambda)$ and Q$(\lambda)$      & $\alpha_q$                                  & $e^{U(\ln 10^{-3},\ln 10^{-1})}$ & $e^{U(\ln 10^{-6},\ln 10^{-3})}$                                    \\
Sarsa$(\lambda)$-s and Q$(\lambda)$-s  & $\alpha_q$                                  & $e^{U(\ln 10^{-3},\ln 10^{-1})}$ & $e^{U(\ln 10^{-3},\ln 10^{0})/|\phi|}$                              \\
\rowcolor[HTML]{EFEFEF} 
All AC, PPO, and NACTD                 & Policy                                      & Tabular Softmax                  & Linear Softmax with Fourier basis                                   \\
AC, NACTD                              & $\alpha_v$                                  & $e^{U(\ln 10^{-3},\ln 10^{-1})}$ & $e^{U(\ln 10^{-6},\ln 10^{-3})}$                                    \\
\rowcolor[HTML]{EFEFEF} 
AC, NACTD                              & $\alpha_p$                                  & $e^{U(\ln 10^{-3},\ln 10^{-1})}$ & $e^{U(\ln 10^{-6},\ln 10^{-3})}$                                    \\
AC-S                                   & $\alpha_v$                                  & $e^{U(\ln 10^{-3},\ln 10^{-1})}$ & $e^{U(\ln 10^{-3},\ln 10^{0})/|\phi|}$                              \\
\rowcolor[HTML]{EFEFEF} 
AC-S, AC-Parl2                         & $\alpha_p$                                  & $e^{U(\ln 10^{-3},\ln 10^{-1})}$ & $e^{U(\ln 10^{-3},\ln 10^{0})/(|\phi| \times \text{num\_actions)}}$ \\
NAC-TD                                 & $\alpha_w$                                  & $e^{U(\ln 10^{-3},\ln 10^{-1})}$ & $e^{U(\ln 10^{-6},\ln 10^{-3})}$                                    \\
\rowcolor[HTML]{EFEFEF} 
NAC-TD                                 & normalize\_gradient                         & \multicolumn{2}{c}{\cellcolor[HTML]{EFEFEF}True}                                                       \\
PPO                                    & clip                                        & $U(0.1,0.3)$                     &                                                                     \\
\rowcolor[HTML]{EFEFEF} 
PPO                                    & entropy\_coef                               & \multicolumn{2}{c}{\cellcolor[HTML]{EFEFEF}$e^{U(\ln 10^{-8},\ln 10^{-2})}$}                           \\
PPO                                    & steps\_per\_batch                           & \multicolumn{2}{c}{$2^{U(\log_2 64, \log_2 256)}$}                                                     \\
\rowcolor[HTML]{EFEFEF} 
PPO                                    & epochs                                      & \multicolumn{2}{c}{\cellcolor[HTML]{EFEFEF}$U(\{1,\dotsc,10\})$}                                       \\
PPO                                    & batch\_size                                 & \multicolumn{2}{c}{$2^{U(\log_2 16, \log_2 \min(64, \text{steps\_per\_batch}))}$}                       \\
\rowcolor[HTML]{EFEFEF} 
PPO                                    & Adam-$\epsilon$                             & \multicolumn{2}{c}{\cellcolor[HTML]{EFEFEF}$10^{-5}$}                                                  \\
PPO                                    & Adam-$\alpha$                               & $e^{U(\ln 10^{-3},\ln 10^{-1})}$ & $e^{U(\ln 10^{-6},\ln 10^{-3})}$                                    \\
\rowcolor[HTML]{EFEFEF} 
PPO                                    & Adam-$\beta_1$                              & \multicolumn{2}{c}{\cellcolor[HTML]{EFEFEF}$0.9$}                                                      \\
PPO                                    & Adam-$\beta_2$                              & \multicolumn{2}{c}{$0.999$}                                                                            \\
\rowcolor[HTML]{EFEFEF} 
Fourier basis                          & dorder                                      & N/A                              & $U(\{0,\dotsc,9\})$                                                 \\
Fourier basis                          & iorder                                      & N/A                              & $U(\{1,\dotsc,9\})$                                                 \\
\rowcolor[HTML]{EFEFEF} 
Fourier basis                          & trig                                        & N/A                              & $\texttt{cos}$                                                     
\end{tabular}
\caption{This table show the distributions from which each hyperparameter is sampled. The All algorithm means the hyperparameter and distribution were used for all algorithms. Steps sizes are labeled with various $\alpha$s. The discount factor $\gamma$ an algorithm uses is scaled down from $\Gamma$ that is specified by the environment. For all environments used in this work $\Gamma = 1.0$. PPO uses the same learning rate for both the policy and value function. The max dependent order on the Fourier basis is limited such that no more than $10,\!000$ features are generated as a result of dorder. }
\label{tab:hyperparametertable}
\end{table}

\section{Environments}
\label{app:environments}
This section describes the environments used in the experiments. 
All environments are listed in Table \ref{tab:envtable}. 
Environments were recreated in the Julia language and many implementations follow closely to the that in the RLPy repository \citep{geramifard2015rlpy}. 
Most environments are best described by either the paper publishing the environment or by examining the source code we provide. We also describe the discrete environments below.

There are eight discrete domains used in the work four Gridworld environments and four chain environments. 
The Gridworld environments are an $N \times N$ grid with the agent starting every episode in the top left corner and goal state in the bottom right. The reward is $-1$ at every step until the goal state is reached, then the episode is then terminated. In every state there are four actions: up, down, left, and right. The transition dynamics are either deterministic, meaning an up action sends the agent up one state unless it is outside the map, or the dynamics are stochastic, meaning the agent might randomly move to one of the states perpendicular to the intended direction or stays in the current state. 
The chain environments are $N$ chains where there are $N$ states each with a connection only to the state directly to the left or right of it and end points only connect to the one state they are next to. The agent starts in state one (far left of the chain) and the goal state is state $N$ (far right of the chain). The reward is $-1$ every step until the goal state is reached and then the episode terminates. Both gridworld and chain MDPs terminate episodes in a finite time based on the size of the problem. Gridworld problems terminate after $20N^2$ steps have been taken and chain environments terminate after $20N$ steps are taken. 

\begin{table}
\centering
\begin{tabular}{lcc}
\hline
\multicolumn{1}{c}{\textbf{Environment}} & \textbf{Num Episodes} & \textbf{State Space} \\ \hline
Gridworld 5 Deterministic                & 100                   & Discrete             \\
Gridworld 5 Stochastic                   & 100                   & Discrete             \\
Gridworld 10 Deterministic               & 100                   & Discrete             \\
Gridworld 10 Stochastic                  & 100                   & Discrete             \\
Chain 10 Deterministic                   & 100                   & Discrete             \\
Chain 10 Stochastic                      & 100                   & Discrete             \\
Chain 50 Deterministic                   & 100                   & Discrete             \\
Chain 50 Stochastic                      & 100                   & Discrete             \\
Acrobot                                  & 500                   & Continuous           \\
Cart-Pole                                & 100                   & Continuous           \\
MountainCar                              & 100                   & Continuous           \\
PinBall Empty                            & 100                   & Continuous           \\
PinBall Box                              & 100                   & Continuous           \\
Pinball Medium                           & 100                   & Continuous           \\
PinBall Single                           & 200                   & Continuous          
\end{tabular}
\caption{This table list every used in this paper along with the number of episodes each algorithm was allowed to interact with the environment and its type of state space. }
\label{tab:envtable}
\end{table}

\section{Alternative Bounding Techniques}
\label{app:altbounds}
As pointed out in our description of PBP, we use the nonparametric concentration inequalities DKW and Anderson's inequalities. These inequalities are often conservative and lead to conservative confidence intervals. So we investigate two alternatives PBP-t a method that replaces DKW and Anderson's inequality with the one based on Students t-distribution and the percentile bootstrap.

In PBP-t everything is the same as PBP except we compute confidence intervals on the mean normalized performance as follows
\begin{equation}
\begin{aligned}
    Z^-_{i,j,k} = & \ \mu_{i,j,k} - \frac{\hat \sigma}{\sqrt{T_{i,j}}} t_{1-\delta',T_{i,j}-1}\\
    Z^+_{i,j,k} = & \ \mu_{i,j,k} + \frac{\hat \sigma}{\sqrt{T_{i,j}}} t_{1-\delta',T_{i,j}-1}\\
\end{aligned}
\end{equation}
where $z_{i,j,k,t} = \hat F_{X_{k,j}}(x_{i,j,t})$, $\mu_{i,j,k} = \frac{1}{T_{i,j}} \sum_{t=1}^{T_{i,j}} z_{i,j,k,t}$, $\hat \sigma = \sqrt{\sum_{t=1}^{T_{i,j}} (\mu_{i,j,k} - z_{i,j,k,t})^2 / (T_{i,j}-1)}$, $t_{1-\delta', \nu}$ is the $100(1-\delta')$ percentile of Student's t-distribution with $\nu$ degrees of freedom, and we set $\delta' = \delta / (|\mathcal{A}||\mathcal{M}|)$. Notice that there are $|\mathcal{A}|^2|\mathcal{M}|$ comparisons being made, so $\delta' = \delta / (|\mathcal{A}|^2|\mathcal{M}|)$ should be used to account for more the multiple comparisons. However, it is likely that if one comparison with an algorithm $i$ fails then that there will be failures with the other $|\mathcal{A}|-1$ algorithms so we use the smaller $\delta'$ as a heuristic for tighter confidence intervals.

In the bootstrap procedure, we use the percentile bootstrap with a confidence level of $\delta' = \delta/(|\mathcal{A}||\mathcal{M}|)$ and $10,\!000$ bootstrap samples of the aggregate performance. Each bootstrap is formed by re-sampling the performance of each algorithm on each environment for the collected data. Then the aggregate performance for each bootstrap is computed. The lower and upper confidence intervals are given by the $100(\delta'/2)$ and $100(1-\delta'/2)$ percentile from the bootstrap aggregate performance. 
Since $\delta'/2$ is often really small, $10,\!000$ bootstrap samples are needed to get confidence intervals that more accurately reflect the true confidence intervals. This requires a substantial amount of compute time and can take over four hours for an original sample size of $10,\!000$.

\section{Confidence interval test experiment}
\label{app:boundsexperiment}
To test the different bounding techniques, we estimate the failure rate of each confidence interval technique at different sample sizes. For this experiment we execute $1,\!000$ trials of the evaluation procedure using samples sizes (trials per algorithm per environment) of $10$, $30$, $100$, $1,\!000$, and $10,\!000$. 
There are a total of $11.14$ million samples per algorithm per environment. To reduce computation costs, we limit this experiment to only include the Sarsa-Parl2, Q-Parl2, AC-Parl2, and Sarsa($\lambda$)-s. Additionally, we reduce the environment set to be the discrete environments and mountain car. Then we compute the proportion of violations for any confidence interval. All methods use the same data, so the results are not independent. We choose not to run independent samples of the bounds to limit our environmental impact.
The method to compute of the proportion of violations and number of significant pairwise comparison can be found in the source code.

It it important to note that when this experiment was run, there was a bug in the code that made the step-size for the Parl2 methods smaller by a factor of $0.1$. This does not invalidate the results, only that the algorithms run are not equivalent those used in the other experiments in this paper. The main impact of this difference was that Sarsa-Parl2 and Q-Parl2 did were not as effective on the discrete MDPs (though they diverged even less) and their scores were nearly the same. Since both of these algorithms had near identical scores on most of the environments, it became almost impossible to differentiate them, so detecting five out of six ($83\%$) statistically significant comparisons is considered optimal for this experiment.

%% file: appendix_performance.tex
\section{Performances}
\label{app:indperf}

This section illustrates the distribution of performance of each algorithm. Tables of the average performance (rounded to the tenths place) along with the algorithm rank on that environment. In the Figures showing the performance distributions the shaded regions represent $100(1-0.05/|\mathcal{A}||\mathcal{M}|)\% \approx 99.9697\%$ confidence intervals computed via DKW.
In the performance tables $100(1-0.05/|\mathcal{A}||\mathcal{M}|)\%$ confidence intervals are shown in parentheses. 

\newpage

\begin{figure}[ht!]
    \centering
    \includegraphics[width=0.6\textwidth]{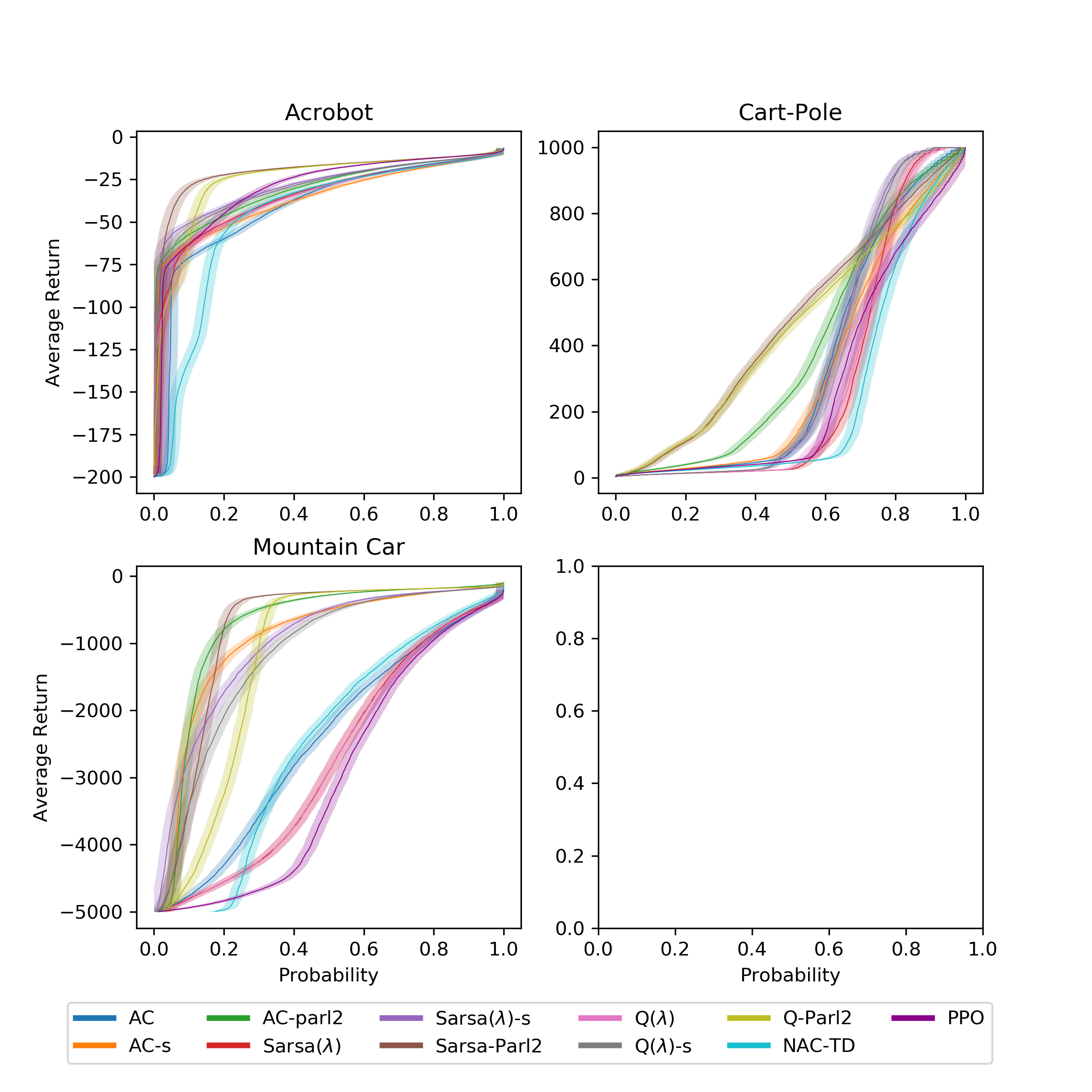}
    \label{fig:classic_dist}
\end{figure}

\begin{figure}[hb!]
    \centering
    \includegraphics[width=0.6\textwidth]{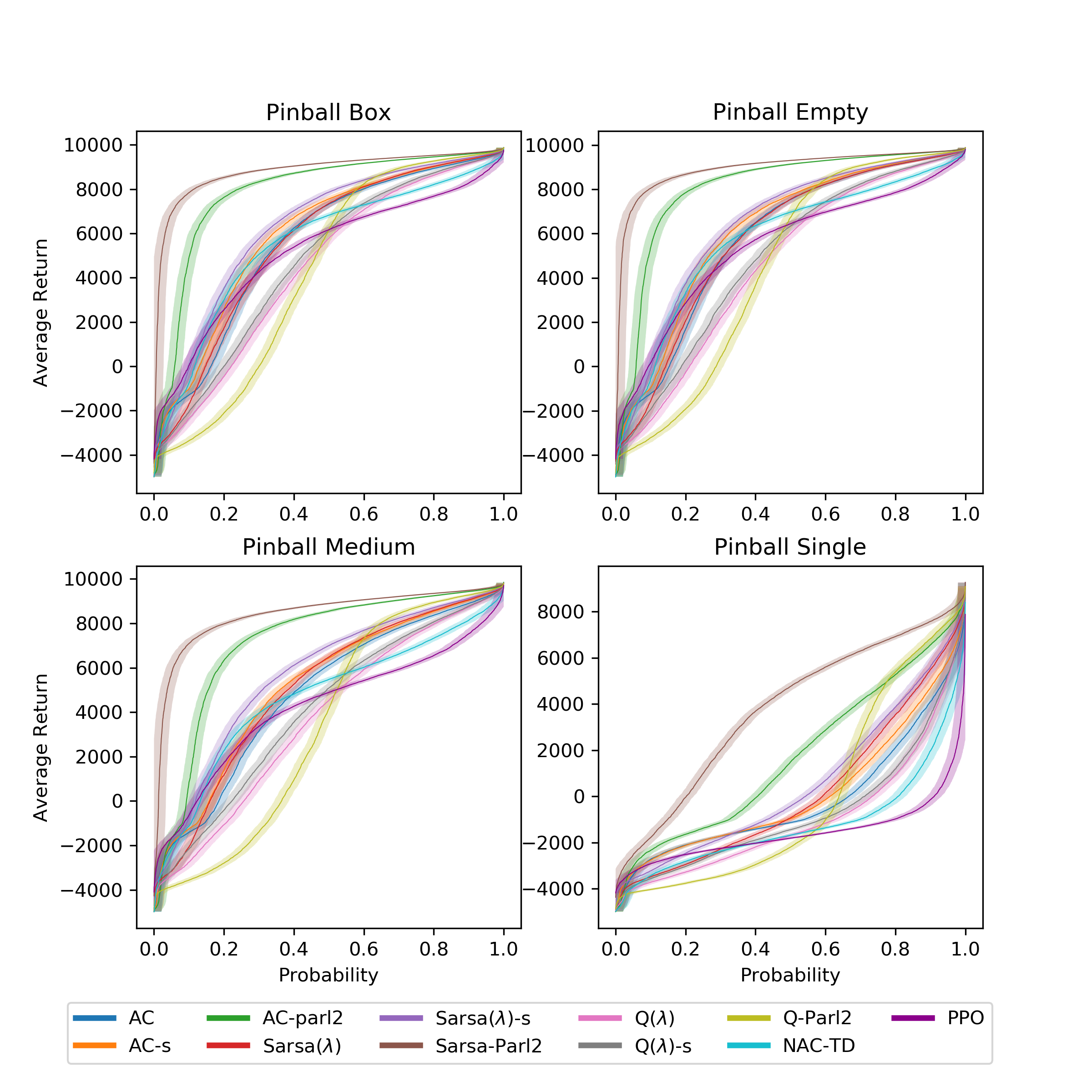}
    \label{fig:pinball_dist}
\end{figure}

\begin{figure}[ht!]
    \centering
    \includegraphics[width=0.6\textwidth]{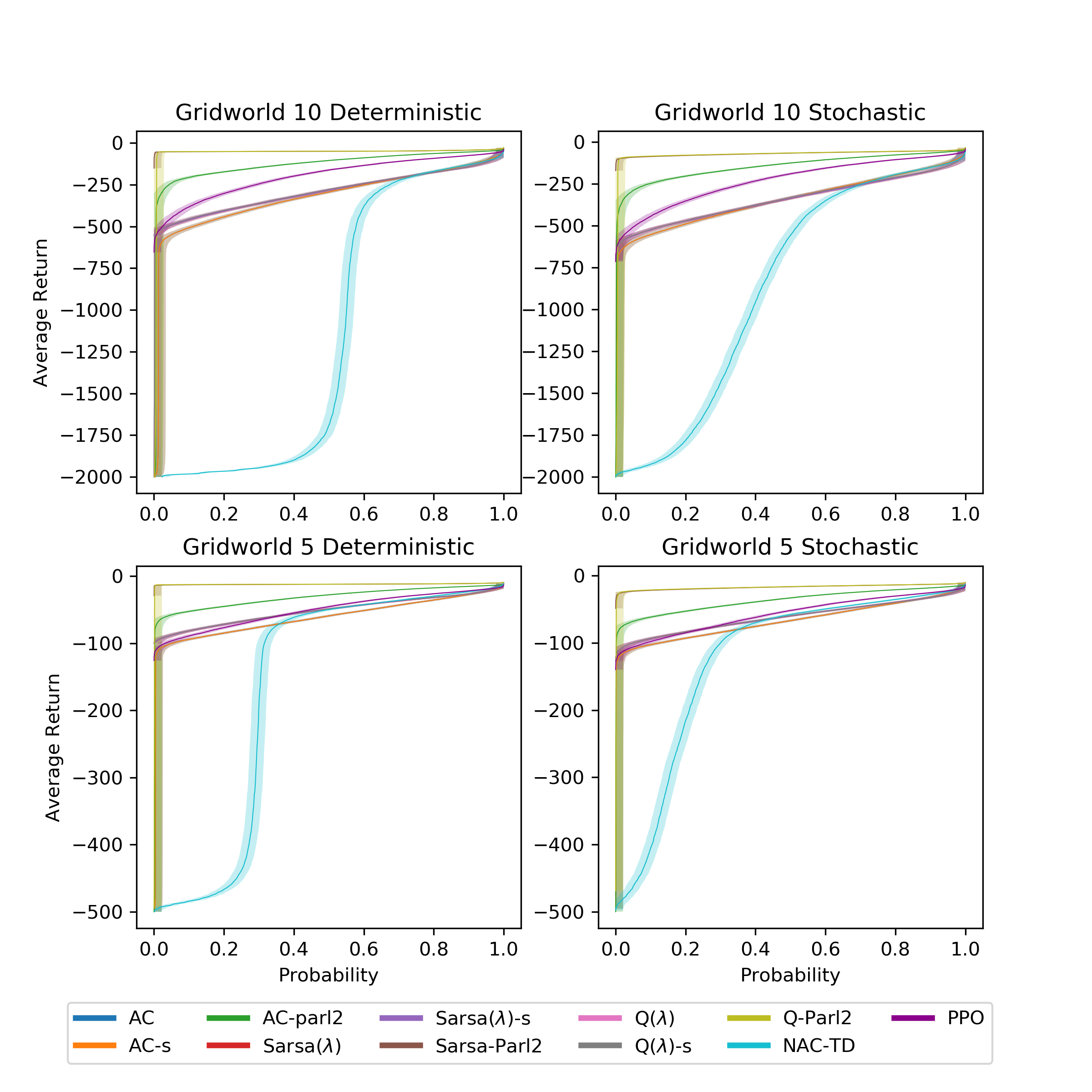}
    \label{fig:gridworld_dist}
\end{figure}

\begin{figure}[hb!]
    \centering
    \includegraphics[width=0.6\textwidth]{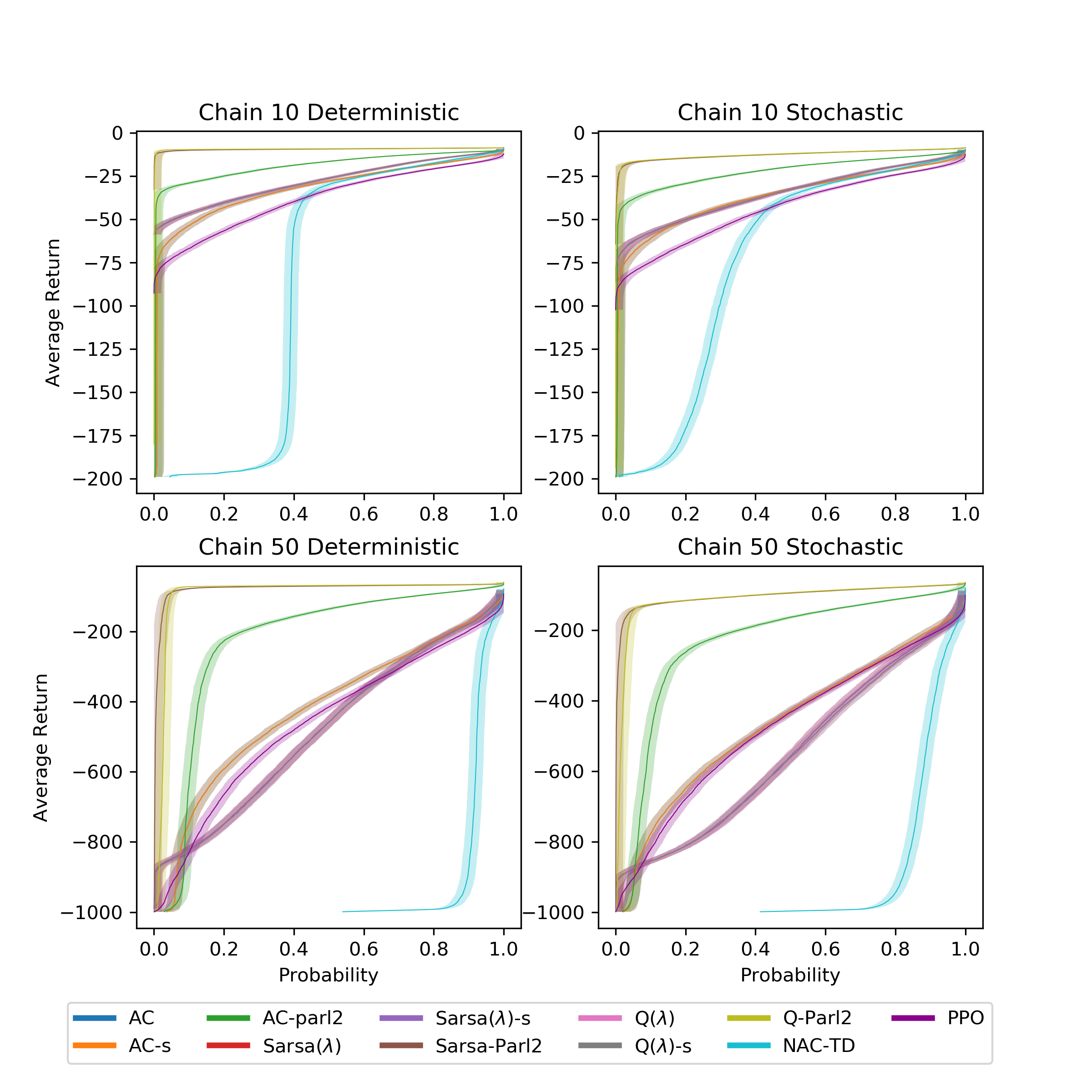}
    \label{fig:chain_dist}
\end{figure}

\newpage

\begin{table}[h!]
\centering 
\begin{tabular}{lcr}
\toprule
\multicolumn{3}{c}{\textbf{Acrobot}}                      \\ 
\midrule 
          Algorithm &                  Mean &         Rank \\
\midrule
\rowcolor[HTML]{EFEFEF} 
        Sarsa-Parl2 &  -20.6 (-24.6, -18.1) &     1 (2, 1) \\
            Q-Parl2 &  -25.7 (-29.8, -22.8) &     2 (6, 1) \\
\rowcolor[HTML]{EFEFEF} 
 Sarsa($\lambda$)-s &  -28.3 (-32.3, -26.3) &     3 (6, 2) \\
     Q($\lambda$)-s &  -30.3 (-34.3, -27.9) &     4 (9, 2) \\
\rowcolor[HTML]{EFEFEF} 
                PPO &  -30.5 (-34.5, -26.7) &     5 (9, 2) \\
           AC-parl2 &  -30.6 (-34.6, -28.8) &     6 (9, 2) \\
\rowcolor[HTML]{EFEFEF} 
   Sarsa($\lambda$) &  -34.7 (-38.7, -32.4) &    7 (10, 4) \\
       Q($\lambda$) &  -35.7 (-39.7, -33.3) &    8 (10, 4) \\
\rowcolor[HTML]{EFEFEF} 
               AC-s &  -36.0 (-40.0, -33.5) &    9 (10, 4) \\
                 AC &  -41.4 (-45.4, -37.3) &   10 (11, 7) \\
\rowcolor[HTML]{EFEFEF} 
             NAC-TD &  -47.1 (-51.1, -43.0) &  11 (11, 10) \\
\bottomrule
\end{tabular}
\label{tab:perf_acrobot}
%
%
%
%
\centering
\begin{tabular}{lcr}
\toprule
\multicolumn{3}{c}{\textbf{Cart-Pole}}                      \\ 
\midrule 
          Algorithm &                  Mean &         Rank \\
\midrule
\rowcolor[HTML]{EFEFEF} 
        Sarsa-Parl2 &  469.2 (448.3, 490.0) &     1 (2, 1) \\
            Q-Parl2 &  450.3 (429.6, 471.1) &     2 (2, 1) \\
\rowcolor[HTML]{EFEFEF} 
           AC-parl2 &  390.4 (369.7, 411.2) &     3 (3, 3) \\
 Sarsa($\lambda$)-s &  347.5 (326.5, 368.4) &     4 (7, 4) \\
\rowcolor[HTML]{EFEFEF} 
     Q($\lambda$)-s &  345.5 (324.5, 366.3) &     5 (7, 4) \\
                 AC &  338.7 (317.8, 359.6) &     6 (7, 4) \\
\rowcolor[HTML]{EFEFEF} 
               AC-s &  320.6 (300.0, 341.5) &     7 (9, 4) \\
       Q($\lambda$) &  291.0 (270.1, 311.9) &    8 (10, 7) \\
\rowcolor[HTML]{EFEFEF} 
   Sarsa($\lambda$) &  287.2 (265.9, 308.6) &    9 (10, 7) \\
                PPO &  276.3 (256.1, 297.1) &   10 (11, 8) \\
\rowcolor[HTML]{EFEFEF} 
             NAC-TD &  244.7 (224.0, 265.5) &  11 (11, 10) \\
\bottomrule
\end{tabular}
\label{tab:perf_cartpole}
\end{table}


\begin{table}[h!] 
\centering 
\begin{tabular}{lcr}
\toprule
\multicolumn{3}{c}{\textbf{Mountain Car}}                      \\ 
\midrule 
          Algorithm &                        Mean &         Rank \\
\midrule
\rowcolor[HTML]{EFEFEF} 
           AC-parl2 &     -791.6 (-894.0, -688.7) &     1 (3, 1) \\
        Sarsa-Parl2 &     -843.2 (-945.0, -740.5) &     2 (4, 1) \\
\rowcolor[HTML]{EFEFEF} 
               AC-s &    -966.0 (-1068.0, -863.1) &     3 (4, 1) \\
 Sarsa($\lambda$)-s &   -1024.2 (-1125.8, -923.5) &     4 (5, 2) \\
\rowcolor[HTML]{EFEFEF} 
     Q($\lambda$)-s &  -1197.3 (-1298.9, -1094.9) &     5 (6, 4) \\
            Q-Parl2 &  -1269.3 (-1371.2, -1166.4) &     6 (6, 5) \\
\rowcolor[HTML]{EFEFEF} 
                 AC &  -2477.1 (-2576.0, -2374.3) &     7 (8, 7) \\
             NAC-TD &  -2489.3 (-2589.1, -2386.4) &     8 (8, 7) \\
\rowcolor[HTML]{EFEFEF} 
   Sarsa($\lambda$) &  -2769.1 (-2867.8, -2666.2) &    9 (10, 9) \\
       Q($\lambda$) &  -2779.3 (-2877.9, -2676.4) &   10 (10, 9) \\
\rowcolor[HTML]{EFEFEF} 
                PPO &  -3026.3 (-3124.9, -2923.5) &  11 (11, 11) \\
\bottomrule
\end{tabular}
\label{tab:perf_mntcar}
\end{table}


\begin{table}[h!] 
\centering 
\begin{tabular}{lcr}
\toprule
\multicolumn{3}{c}{\textbf{Chain 10 Deterministic}}                      \\ 
\midrule 
          Algorithm &                  Mean &         Rank \\
\midrule
\rowcolor[HTML]{EFEFEF} 
            Q-Parl2 &    -9.2 (-13.2, -9.1) &     1 (2, 1) \\
        Sarsa-Parl2 &    -9.3 (-13.3, -9.2) &     2 (2, 1) \\
\rowcolor[HTML]{EFEFEF} 
           AC-parl2 &  -19.0 (-23.0, -17.7) &     3 (3, 3) \\
   Sarsa($\lambda$) &  -27.9 (-31.8, -26.9) &     4 (9, 4) \\
\rowcolor[HTML]{EFEFEF} 
 Sarsa($\lambda$)-s &  -27.9 (-31.8, -26.9) &     4 (9, 4) \\
       Q($\lambda$) &  -27.9 (-31.9, -26.9) &     6 (9, 4) \\
\rowcolor[HTML]{EFEFEF} 
     Q($\lambda$)-s &  -27.9 (-31.9, -26.9) &     6 (9, 4) \\
                 AC &  -32.2 (-36.2, -29.9) &     8 (9, 4) \\
\rowcolor[HTML]{EFEFEF} 
               AC-s &  -32.2 (-36.2, -29.9) &     8 (9, 4) \\
                PPO &  -38.0 (-41.9, -36.5) &  10 (10, 10) \\
\rowcolor[HTML]{EFEFEF} 
             NAC-TD &  -89.7 (-93.7, -85.7) &  11 (11, 11) \\
\bottomrule
\end{tabular}
%
%
%
%
\centering 
\begin{tabular}{lcr}
\toprule
\multicolumn{3}{c}{\textbf{Chain 10 Stochastic}}                      \\ 
\midrule 
          Algorithm &                  Mean &         Rank \\
\midrule
\rowcolor[HTML]{EFEFEF} 
            Q-Parl2 &  -12.7 (-16.7, -12.2) &     1 (2, 1) \\
        Sarsa-Parl2 &  -12.8 (-16.8, -12.4) &     2 (2, 1) \\
\rowcolor[HTML]{EFEFEF} 
           AC-parl2 &  -22.6 (-26.6, -21.2) &     3 (3, 3) \\
       Q($\lambda$) &  -34.9 (-38.9, -33.6) &     4 (9, 4) \\
\rowcolor[HTML]{EFEFEF} 
     Q($\lambda$)-s &  -34.9 (-38.9, -33.6) &     4 (9, 4) \\
   Sarsa($\lambda$) &  -35.1 (-39.0, -33.7) &     6 (9, 4) \\
\rowcolor[HTML]{EFEFEF} 
 Sarsa($\lambda$)-s &  -35.1 (-39.0, -33.7) &     6 (9, 4) \\
                 AC &  -36.9 (-40.8, -34.7) &     8 (9, 4) \\
\rowcolor[HTML]{EFEFEF} 
               AC-s &  -36.9 (-40.8, -34.7) &     8 (9, 4) \\
                PPO &  -43.8 (-47.6, -42.1) &  10 (10, 10) \\
\rowcolor[HTML]{EFEFEF} 
             NAC-TD &  -75.2 (-79.1, -71.2) &  11 (11, 11) \\
\bottomrule
\end{tabular}
\label{tab:perf_chain10s}
\end{table}


\begin{table}[h!] 
\centering 
\begin{tabular}{lcr}
\toprule
\multicolumn{3}{c}{\textbf{Chain 50 Deterministic}}                      \\ 
\midrule 
          Algorithm &                     Mean &         Rank \\
\midrule
\rowcolor[HTML]{EFEFEF} 
        Sarsa-Parl2 &     -78.3 (-97.9, -71.5) &     1 (2, 1) \\
            Q-Parl2 &    -92.7 (-112.3, -74.6) &     2 (2, 1) \\
\rowcolor[HTML]{EFEFEF} 
           AC-parl2 &  -231.0 (-250.5, -211.3) &     3 (3, 3) \\
                 AC &  -425.5 (-444.1, -405.8) &     4 (6, 4) \\
\rowcolor[HTML]{EFEFEF} 
               AC-s &  -425.5 (-444.1, -405.8) &     4 (6, 4) \\
                PPO &  -462.2 (-480.4, -442.6) &    6 (10, 4) \\
\rowcolor[HTML]{EFEFEF} 
   Sarsa($\lambda$) &  -479.5 (-498.3, -462.4) &    7 (10, 6) \\
 Sarsa($\lambda$)-s &  -479.5 (-498.3, -462.4) &    7 (10, 6) \\
\rowcolor[HTML]{EFEFEF} 
       Q($\lambda$) &  -481.6 (-500.4, -464.0) &    9 (10, 6) \\
     Q($\lambda$)-s &  -481.6 (-500.4, -464.0) &    9 (10, 6) \\
\rowcolor[HTML]{EFEFEF} 
             NAC-TD &  -928.4 (-946.8, -908.7) &  11 (11, 11) \\
\bottomrule
\end{tabular}
%
%
%
%
\centering 
\begin{tabular}{lcr}
\toprule
\multicolumn{3}{c}{\textbf{Chain 50 Stochastic}}                      \\ 
\midrule 
          Algorithm &                     Mean &         Rank \\
\midrule
\rowcolor[HTML]{EFEFEF} 
        Sarsa-Parl2 &   -101.6 (-121.1, -96.4) &     1 (2, 1) \\
            Q-Parl2 &   -111.1 (-130.6, -97.5) &     2 (2, 1) \\
\rowcolor[HTML]{EFEFEF} 
           AC-parl2 &  -239.3 (-258.7, -219.7) &     3 (3, 3) \\
                 AC &  -463.2 (-481.5, -443.6) &     4 (6, 4) \\
\rowcolor[HTML]{EFEFEF} 
               AC-s &  -463.2 (-481.5, -443.6) &     4 (6, 4) \\
                PPO &  -474.4 (-492.2, -455.3) &     6 (6, 4) \\
\rowcolor[HTML]{EFEFEF} 
   Sarsa($\lambda$) &  -546.9 (-565.2, -529.2) &    7 (10, 7) \\
 Sarsa($\lambda$)-s &  -546.9 (-565.2, -529.2) &    7 (10, 7) \\
\rowcolor[HTML]{EFEFEF} 
       Q($\lambda$) &  -550.3 (-568.7, -532.5) &    9 (10, 7) \\
     Q($\lambda$)-s &  -550.3 (-568.7, -532.5) &    9 (10, 7) \\
\rowcolor[HTML]{EFEFEF} 
             NAC-TD &  -897.4 (-915.3, -877.8) &  11 (11, 11) \\
\bottomrule
\end{tabular}
\label{tab:perf_chain50s}
\end{table}


\begin{table}[h!] 
\centering 
\begin{tabular}{lcr}
\toprule
\multicolumn{3}{c}{\textbf{Gridworld 10 Deterministic}}                      \\ 
\midrule 
          Algorithm &                        Mean &         Rank \\
\midrule
\rowcolor[HTML]{EFEFEF} 
        Sarsa-Parl2 &        -49.5 (-90.7, -48.9) &     1 (2, 1) \\
            Q-Parl2 &       -60.8 (-102.0, -48.5) &     2 (2, 1) \\
\rowcolor[HTML]{EFEFEF} 
           AC-parl2 &     -130.6 (-171.6, -115.1) &     3 (3, 3) \\
                PPO &     -199.6 (-240.4, -188.8) &     4 (4, 4) \\
\rowcolor[HTML]{EFEFEF} 
       Q($\lambda$) &     -290.6 (-331.0, -278.9) &    5 (10, 5) \\
     Q($\lambda$)-s &     -290.6 (-331.0, -278.9) &    5 (10, 5) \\
\rowcolor[HTML]{EFEFEF} 
   Sarsa($\lambda$) &     -291.8 (-332.3, -281.3) &    7 (10, 5) \\
 Sarsa($\lambda$)-s &     -291.8 (-332.3, -281.3) &    7 (10, 5) \\
\rowcolor[HTML]{EFEFEF} 
                 AC &     -324.7 (-365.2, -295.3) &    9 (10, 5) \\
               AC-s &     -324.7 (-365.2, -295.3) &    9 (10, 5) \\
\rowcolor[HTML]{EFEFEF} 
             NAC-TD &  -1141.3 (-1181.6, -1099.9) &  11 (11, 11) \\
\bottomrule
\end{tabular}
%
%
%
%
\centering 
\begin{tabular}{lcr}
\toprule
\multicolumn{3}{c}{\textbf{Gridworld 10 Stochastic}}                      \\ 
\midrule 
          Algorithm &                     Mean &         Rank \\
\midrule
\rowcolor[HTML]{EFEFEF} 
        Sarsa-Parl2 &    -67.5 (-108.6, -66.1) &     1 (2, 1) \\
            Q-Parl2 &    -75.4 (-116.4, -64.8) &     2 (2, 1) \\
\rowcolor[HTML]{EFEFEF} 
           AC-parl2 &  -146.4 (-187.3, -135.8) &     3 (3, 3) \\
                PPO &  -229.5 (-270.1, -217.6) &     4 (4, 4) \\
\rowcolor[HTML]{EFEFEF} 
       Q($\lambda$) &  -339.2 (-379.2, -326.9) &    5 (10, 5) \\
     Q($\lambda$)-s &  -339.2 (-379.2, -326.9) &    5 (10, 5) \\
\rowcolor[HTML]{EFEFEF} 
   Sarsa($\lambda$) &  -342.3 (-382.4, -330.0) &    7 (10, 5) \\
 Sarsa($\lambda$)-s &  -342.3 (-382.4, -330.0) &    7 (10, 5) \\
\rowcolor[HTML]{EFEFEF} 
                 AC &  -347.2 (-387.4, -329.7) &    9 (10, 5) \\
               AC-s &  -347.2 (-387.4, -329.7) &    9 (10, 5) \\
\rowcolor[HTML]{EFEFEF} 
             NAC-TD &  -864.1 (-904.2, -823.3) &  11 (11, 11) \\
\bottomrule
\end{tabular}
\label{tab:perf_gw10s}
\end{table}


\begin{table}[h!] 
\centering 
\begin{tabular}{lcr}
\toprule
\multicolumn{3}{c}{\textbf{Gridworld 5 Deterministic}}                      \\ 
\midrule 
          Algorithm &                     Mean &         Rank \\
\midrule
\rowcolor[HTML]{EFEFEF} 
        Sarsa-Parl2 &     -12.2 (-22.5, -12.1) &     1 (2, 1) \\
            Q-Parl2 &     -12.5 (-22.7, -12.0) &     2 (2, 1) \\
\rowcolor[HTML]{EFEFEF} 
           AC-parl2 &     -32.2 (-42.4, -30.1) &     3 (3, 3) \\
                PPO &     -51.1 (-61.2, -49.0) &    4 (10, 4) \\
\rowcolor[HTML]{EFEFEF} 
       Q($\lambda$) &     -51.4 (-61.5, -49.6) &    5 (10, 4) \\
     Q($\lambda$)-s &     -51.4 (-61.5, -49.6) &    5 (10, 4) \\
\rowcolor[HTML]{EFEFEF} 
   Sarsa($\lambda$) &     -51.9 (-62.0, -50.1) &    7 (10, 4) \\
 Sarsa($\lambda$)-s &     -51.9 (-62.0, -50.1) &    7 (10, 4) \\
\rowcolor[HTML]{EFEFEF} 
                 AC &     -62.0 (-72.1, -58.2) &    9 (10, 4) \\
               AC-s &     -62.0 (-72.1, -58.2) &    9 (10, 4) \\
\rowcolor[HTML]{EFEFEF} 
             NAC-TD &  -168.1 (-178.3, -157.9) &  11 (11, 11) \\
\bottomrule
\end{tabular}
%
%
%
%
\centering 
\begin{tabular}{lcr}
\toprule
\multicolumn{3}{c}{\textbf{Gridworld 5 Stochastic}}                      \\ 
\midrule 
          Algorithm &                     Mean &         Rank \\
\midrule
\rowcolor[HTML]{EFEFEF} 
        Sarsa-Parl2 &     -17.0 (-27.2, -16.6) &     1 (2, 1) \\
            Q-Parl2 &     -17.1 (-27.4, -16.3) &     2 (2, 1) \\
\rowcolor[HTML]{EFEFEF} 
           AC-parl2 &     -37.2 (-47.3, -35.4) &     3 (3, 3) \\
                PPO &     -57.2 (-67.3, -54.9) &    4 (10, 4) \\
\rowcolor[HTML]{EFEFEF} 
       Q($\lambda$) &     -61.5 (-71.6, -59.4) &    5 (10, 4) \\
     Q($\lambda$)-s &     -61.5 (-71.6, -59.4) &    5 (10, 4) \\
\rowcolor[HTML]{EFEFEF} 
   Sarsa($\lambda$) &     -61.5 (-71.6, -59.4) &    7 (10, 4) \\
 Sarsa($\lambda$)-s &     -61.5 (-71.6, -59.4) &    7 (10, 4) \\
\rowcolor[HTML]{EFEFEF} 
                 AC &     -67.6 (-77.6, -64.8) &    9 (10, 4) \\
               AC-s &     -67.6 (-77.6, -64.8) &    9 (10, 4) \\
\rowcolor[HTML]{EFEFEF} 
             NAC-TD &  -125.4 (-135.5, -115.4) &  11 (11, 11) \\
\bottomrule
\end{tabular}
\label{tab:perf_gw5s}
\end{table}


\begin{table}[h!]
\centering 
\begin{tabular}{lcr}
\toprule
\multicolumn{3}{c}{\textbf{Pinball Box}}                      \\ 
\midrule 
          Algorithm &                     Mean &        Rank \\
\midrule
\rowcolor[HTML]{EFEFEF} 
        Sarsa-Parl2 &  8823.2 (8513.4, 8998.4) &    1 (1, 1) \\
           AC-parl2 &  7875.4 (7565.9, 8170.2) &    2 (2, 2) \\
\rowcolor[HTML]{EFEFEF} 
 Sarsa($\lambda$)-s &  6288.3 (5980.0, 6569.4) &    3 (4, 3) \\
               AC-s &  5961.2 (5653.1, 6251.1) &    4 (7, 3) \\
\rowcolor[HTML]{EFEFEF} 
   Sarsa($\lambda$) &  5603.0 (5295.1, 5889.7) &    5 (8, 4) \\
                 AC &  5602.5 (5295.1, 5886.8) &    6 (8, 4) \\
\rowcolor[HTML]{EFEFEF} 
             NAC-TD &  5546.8 (5243.2, 5838.0) &    7 (8, 4) \\
                PPO &  5184.6 (4882.6, 5447.4) &    8 (9, 5) \\
\rowcolor[HTML]{EFEFEF} 
     Q($\lambda$)-s &  4728.6 (4421.8, 5020.1) &   9 (11, 8) \\
       Q($\lambda$) &  4449.9 (4143.3, 4740.8) &  10 (11, 9) \\
\rowcolor[HTML]{EFEFEF} 
            Q-Parl2 &  4246.5 (3937.2, 4539.9) &  11 (11, 9) \\
\bottomrule
\end{tabular}
%
%
%
%
\centering 
\begin{tabular}{lcr}
\toprule
\multicolumn{3}{c}{\textbf{Pinball Empty}}                      \\ 
\midrule 
          Algorithm &                     Mean &        Rank \\
\midrule
\rowcolor[HTML]{EFEFEF} 
        Sarsa-Parl2 &  8942.1 (8631.8, 9115.3) &    1 (1, 1) \\
           AC-parl2 &  8041.0 (7730.8, 8333.4) &    2 (2, 2) \\
\rowcolor[HTML]{EFEFEF} 
 Sarsa($\lambda$)-s &  6382.4 (6073.4, 6664.1) &    3 (5, 3) \\
               AC-s &  6155.8 (5846.5, 6444.9) &    4 (7, 3) \\
\rowcolor[HTML]{EFEFEF} 
                 AC &  5814.2 (5505.4, 6097.3) &    5 (8, 3) \\
   Sarsa($\lambda$) &  5778.6 (5469.8, 6063.7) &    6 (8, 4) \\
\rowcolor[HTML]{EFEFEF} 
             NAC-TD &  5710.1 (5405.2, 5998.0) &    7 (8, 4) \\
                PPO &  5401.3 (5097.3, 5666.6) &    8 (9, 5) \\
\rowcolor[HTML]{EFEFEF} 
     Q($\lambda$)-s &  4891.8 (4584.1, 5182.5) &   9 (11, 8) \\
       Q($\lambda$) &  4602.3 (4294.8, 4892.1) &  10 (11, 9) \\
\rowcolor[HTML]{EFEFEF} 
            Q-Parl2 &  4487.9 (4178.3, 4781.5) &  11 (11, 9) \\
\bottomrule
\end{tabular}
\label{tab:perf_pbempt}
\end{table}


\begin{table}[h!]
\centering 
\begin{tabular}{lcr}
\toprule
\multicolumn{3}{c}{\textbf{Pinball Medium}}                      \\ 
\midrule 
          Algorithm &                     Mean &         Rank \\
\midrule
\rowcolor[HTML]{EFEFEF} 
        Sarsa-Parl2 &  8402.1 (8094.2, 8625.5) &     1 (1, 1) \\
           AC-parl2 &  7128.1 (6820.6, 7429.2) &     2 (2, 2) \\
\rowcolor[HTML]{EFEFEF} 
 Sarsa($\lambda$)-s &  5667.4 (5361.4, 5949.4) &     3 (4, 3) \\
               AC-s &  5195.5 (4890.3, 5487.2) &     4 (7, 3) \\
\rowcolor[HTML]{EFEFEF} 
   Sarsa($\lambda$) &  5050.5 (4745.1, 5336.5) &     5 (7, 4) \\
                 AC &  4835.7 (4531.8, 5123.8) &     6 (7, 4) \\
\rowcolor[HTML]{EFEFEF} 
             NAC-TD &  4652.5 (4353.9, 4940.2) &     7 (9, 4) \\
                PPO &  4227.1 (3932.0, 4493.3) &    8 (10, 7) \\
\rowcolor[HTML]{EFEFEF} 
     Q($\lambda$)-s &  4084.7 (3780.6, 4375.1) &    9 (10, 7) \\
       Q($\lambda$) &  3689.6 (3386.2, 3981.9) &   10 (11, 8) \\
\rowcolor[HTML]{EFEFEF} 
            Q-Parl2 &  3404.5 (3097.9, 3699.2) &  11 (11, 10) \\
\bottomrule
\end{tabular}
%
%
%
%
\centering 
\begin{tabular}{lcr}
\toprule
\multicolumn{3}{c}{\textbf{Pinball Single}}                      \\ 
\midrule 
          Algorithm &                        Mean &         Rank \\
\midrule
\rowcolor[HTML]{EFEFEF} 
        Sarsa-Parl2 &     3754.6 (3470.3, 4023.3) &     1 (1, 1) \\
           AC-parl2 &     1696.7 (1418.4, 1989.3) &     2 (2, 2) \\
\rowcolor[HTML]{EFEFEF} 
 Sarsa($\lambda$)-s &        514.1 (240.8, 793.2) &     3 (5, 3) \\
   Sarsa($\lambda$) &       119.2 (-151.7, 408.1) &     4 (7, 3) \\
\rowcolor[HTML]{EFEFEF} 
               AC-s &       103.8 (-158.5, 389.6) &     5 (7, 3) \\
            Q-Parl2 &       -73.4 (-352.7, 214.7) &     6 (7, 4) \\
\rowcolor[HTML]{EFEFEF} 
                 AC &       -197.1 (-451.4, 95.1) &     7 (7, 4) \\
     Q($\lambda$)-s &    -746.6 (-1005.9, -457.2) &    8 (10, 8) \\
\rowcolor[HTML]{EFEFEF} 
       Q($\lambda$) &    -986.1 (-1244.4, -693.8) &    9 (10, 8) \\
             NAC-TD &   -1229.4 (-1463.7, -934.7) &   10 (11, 8) \\
\rowcolor[HTML]{EFEFEF} 
                PPO &  -1602.2 (-1789.1, -1327.7) &  11 (11, 10) \\
\bottomrule
\end{tabular}
\label{tab:perf_pbsingle}
\end{table}
